\documentclass[lettersize,journal]{IEEEtran}
\usepackage{multirow}
\usepackage{float}
\usepackage{enumitem}
\usepackage{amsthm}
\usepackage{amsmath}
\usepackage{amsfonts}

\newtheorem{proposition}{Proposition}
\usepackage{graphicx}
\usepackage{algorithm}
\usepackage{algorithmic}
\usepackage{subcaption}
\usepackage{textcomp}
\usepackage{stfloats}
\usepackage{xcolor}
\usepackage{hyperref}
\usepackage{makecell}
\usepackage{csquotes}
\usepackage{tabularx}
\theoremstyle{remark}
\usepackage{bm}
\newcommand{\1}{\mathbf{1}}

\begin{document}
\title{Efficient Curvature-aware Graph Network 
}

\author{
        Chaoqun~Fei,
        Tinglve~Zhou,
        Tianyong~Hao,
        Yangyang~Li\\
        \thanks{C.Q. Fei and T.L. Zhou are with the School of Artificial Intelligence, South China Normal University, Foshan 528225, China (e-mail: cqfei@m.scnu.edu.cn; tingluezhou@m.scnu.edu.cn)}
        \thanks{T.Y. Hao is with the School of Computer Science, South China Normal Univerisity, Guangzhou 510631, China (e-mail: haoty@126.com)}
        \thanks{Y.Y. Li is with the State Key Laboratory of Mathematical Sciences, Academy of Mathematics and Systems Science, Chinese Academy of Sciences, Beijing 100190, China (e-mail: yyli@amss.ac.cn)
        \emph{ (Corresponding author: Yangyang Li.)}}
}

\markboth{Journal of \LaTeX\ Class Files,~Vol.~14, No.~8, August~2021}%
{Shell \MakeLowercase{\textit{et al.}}: A Sample Article Using IEEEtran.cls for IEEE Journals}

\maketitle

\begin{abstract}
Graph curvature provides geometric priors for Graph Neural Networks (GNNs), enhancing their ability to model complex graph structures, particularly in terms of structural awareness, robustness, and theoretical interpretability. Among existing methods, Ollivier-Ricci curvature has been extensively studied due to its strong geometric interpretability, effectively characterizing the local geometric distribution between nodes. However, its prohibitively high computational complexity limits its applicability to large-scale graph datasets. To address this challenge, we propose a novel graph curvature measure—Effective Resistance Curvature—which quantifies the ease of message passing along graph edges using the effective resistance between node pairs, instead of the optimal transport distance. This method significantly outperforms Ollivier-Ricci curvature in computational efficiency while preserving comparable geometric expressiveness. Theoretically, we prove the low computational complexity of effective resistance curvature and establish its substitutability for Ollivier-Ricci curvature. Furthermore, extensive experiments on diverse GNN tasks demonstrate that our method achieves competitive performance with Ollivier-Ricci curvature while drastically reducing computational overhead.
\end{abstract}

\section{Introduction}
\IEEEPARstart{G}{raph} Neural Networks (GNNs) are a class of deep learning models specifically designed to process graph-structured data\cite{Joan2013}. Unlike traditional neural networks, GNNs are capable of effectively capturing the relationships between nodes and the global structural information of the graph\cite{defferrard2016}. Through a message-passing mechanism, GNNs allow each node to aggregate information from its neighbors, gradually updating its own representation\cite{Xu2019}. This aggregation process can be iterated multiple times, enabling the network to learn deeper, more abstract features\cite{Henaff2015}\cite{Xu2019}. Due to their flexibility and powerful representational capabilities, GNNs have become an important research direction in machine learning and data mining. Key tasks include node classification\cite{xiao2022graph}, edge classification\cite{cai2020}\cite{gong2019}, and graph classification. Additionally, GNNs can be employed for graph generation\cite{guo2022}, creating new graph structures; graph regression, predicting numerical attributes of graphs or nodes; and graph embedding\cite{cai2018}, which maps graph data to a lower-dimensional space for more efficient analysis. These capabilities make GNNs highly applicable in diverse domains, including social networks\cite{ni2019community}, recommendation systems, biological network analysis, and knowledge graphs.

However, similar to other deep learning models, Graph Neural Networks (GNNs) face significant challenges in terms of interpretability. Due to their "black-box" nature, understanding and explaining the decision-making process of GNNs has become a key area of research. First, \textbf{the complexity of the message-passing mechanism} represents a major obstacle. The representation at each layer is influenced by many neighboring nodes, making it difficult to intuitively grasp the model's reasoning, such as the occurrence of "Over-smoothing"\cite{chen2020} and "Over-squashing" phenomena\cite{akansha2025}. Second, \textbf{the non-Euclidean structure of graph} introduces additional complexity. Graphs intricate non-grid structures and the complex interactions between topology and node features challenge traditional interpretation methods\cite{bronstein2017geometric}.

To address these challenges, researchers have focused on analyzing the influence of neighboring nodes on the final representation. For example, Graph Attention Networks (GATs)\cite{velickovic2017} provide interpretability by using attention weights to highlight the interactions between nodes. By analyzing attention scores across different edges, it becomes possible to identify which edges most significantly impact the final prediction. However, since the attention matrix is learned during model optimization, it provides only a limited degree of interpretability. For more general GNN models, efforts have also been made to incorporate the non-Euclidean geometric structure of the graph, quantify relationships between nodes and their local neighbors, and track the evolution of the learning process\cite{bronstein2017geometric}. These approaches offer new insights into the decision-making process of GNNs and have accelerated the development of interpretable graph neural networks.

Currently, Ollivier-Ricci curvature\cite{ollivier2009ricci}, as a type of discrete non-Euclidean graph structure, is widely applied in graph representation learning\cite{ye2019}. Rooted in the concept of abstract Ricci curvature, it reflects the geometric and topological properties of graphs. Researchers have used Ollivier-Ricci graph curvature to study the "Over-smoothing" and "Over-squashing" phenomena in GNNs\cite{nguyen2023}. They found that \textbf{positive curvature} on graph edges is related to the "Over-smoothing" phenomenon, while \textbf{negative curvature} is associated with "Over-squashing"\cite{alon2020}\cite{topping2021}.  In geometry, the edge curvature of a graph can be interpreted as the degree of variation in the way nodes are connected, quantifying changes in the adjacency matrix of the graph, such as the addition or removal of edges, or changes in edge weights. Ollivier-Ricci curvature primarily measures edge variation by considering the Wasserstein distance between nodes and their probability distributions. Based on Wasserstein distance, it assesses the relative positions and similarities between nodes, reflecting the geometric characteristics of the graph. Locally, Ollivier-Ricci curvature can serve as an indicator of whether the region is "flat." Positive curvature suggests an outward bulge, reflecting stronger inter-node connections and revealing cluster-like structures, which can contribute to "over-smoothing" in GNNs. In contrast, negative curvature indicates an inward concavity, pointing to weaker connections between nodes and highlighting bottleneck structures, thereby helping to explain the phenomenon of "over-squashing".

Although Ollivier-Ricci curvature has shown many advantages in graph analysis and representation learning, it also faces several challenges that limit its application in the field of graph representation learning:
\begin{enumerate}
	\item \textbf{High computational cost:} The computation of Ollivier-Ricci curvature depends on evaluating the probability distributions of nodes and Wasserstein distance, which can lead to significant computational costs, especially in large-scale graphs. This issue becomes more pronounced when dealing with complex graphs.
	\item \textbf{Increased model complexity:} Incorporating Ollivier-Ricci curvature into models like Graph Neural Networks (GNNs) can considerably increase model complexity, which may affect both the interpretability and debug ability of the model.
\end{enumerate}

To address this challenge, we introduce a novel and efficient graph curvature concept-\textbf{Effective Resistance Curvature}- for Graph Neural Network\cite{ellens2011}. This method models the graph as an electrical circuit, where resistance reflects the ease or difficulty of current flow, thereby quantifying the connectivity between edges. We demonstrate that the proposed resistance curvature not only captures the geometric structure of edges in a manner consistent with Ollivier-Ricci curvature, but also significantly reduces computational complexity, making it particularly suitable for large-scale graph applications.

From the perspective of message passing in Graph Neural Network, we analogize information to electric current, likening the propagation of messages across the graph to the flow of current in a circuit. In high-curvature regions, nodes are more densely interconnected, analogous to a network with multiple parallel pathways that exhibit lower effective resistance, thereby facilitating the smooth flow of current—or information. In contrast, low-curvature regions correspond to higher effective resistance, which impedes current flow. This electrical model provides an alternative yet equivalent representation of edge curvature in graphs, offering intuitive insights similar to those of Ollivier-Ricci curvature, while enhancing computational efficiency. The main contributions of this paper are as follows:
\begin{enumerate}
	\item We propose an estimation method for effective resistance curvature and analyze its time complexity, which is significantly lower than that of Ollivier-Ricci curvature.
	\item We provide a theoretical analysis showing that the effective resistance curvature and Ollivier-Ricci curvature describe the node distribution structures on the graph in a fundamentally consistent manner.
	\item In the experimental section, we apply effective resistance curvature to graph representation learning tasks and compare it with traditional Ollivier-Ricci curvature. We demonstrate that, under the advantage of lower complexity, the performance of effective resistance is comparable to that of the Ollivier-Ricci curvature.
\end{enumerate}

\section{Related Works}
\subsection{Graph Curvature}
The study of graph curvature originated from the idea of incorporating the geometric properties of manifolds into graph structure analysis, utilizing tools from geometry and topology to describe the structural features of graphs\cite{Lin2011}\cite{sandhu2015}. 

Ollivier-Ricci curvature (OR curvature) is one of the most well-known and widely used definitions in the study of graph curvature\cite{ollivier2009ricci}. Inspired by Ricci curvature on manifolds, OR curvature defines a curvature measure between two nodes in a graph based on optimal transport theory. It quantifies the difference between the distance between two nodes in the graph and the distance computed via optimal transport, thus revealing the geometric relationships between nodes and the local structural properties of the graph\cite{ollivier2010}. Research has shown that Ollivier-Ricci curvature has significant applications in graph-related learning tasks.
Forman curvature, introduced by mathematician Ronald Forman in 2003, is a high-efficiency curvature measure on graphs\cite{sreejith2016}. It adapts the traditional concept of curvature to the structure of graphs, providing an effective method for analyzing their local geometric properties\cite{weber2017}. Unlike Ollivier-Ricci curvature and other optimal transport-based curvature measures, Forman curvature defines curvature using the simpler structural elements of graphs, such as nodes and edges. However, compared to Forman curvature, Ollivier-Ricci curvature possesses stronger geometric expressiveness and interpretability, making it more suitable for capturing non-Euclidean structure\cite{samal2018}.

\subsection{Curvature-aware Graph Representation Learning}
Recent years have witnessed a growing body of work applying graph curvature to graph neural network (GNN) learning. Among various forms of graph curvature, Ollivier–Ricci curvature stands out for its geometric richness and its ability to capture the structure of node sets more accurately. Studies incorporating Ollivier–Ricci curvature into GNNs primarily focus on leveraging geometric and topological relationships between nodes to enhance performance in graph learning tasks. Ollivier–Ricci curvature has proven particularly valuable in applications such as graph segmentation, community detection, and graph clustering. For instance, Ni et al. (20019) used Ollivier–Ricci curvature to analyze community structures in social networks\cite{ni2019community}. In \cite{sia2019}, it is employed to effectively capture the geometric curvature and structural features of different communities, thereby improving the accuracy of community detection. Another study\cite{ye2019} integrates Ollivier–Ricci curvature into GNNs and proposes using geometric metrics to optimize node representations. Meanwhile, \cite{li2022curvature} introduces a GNN architecture that incorporates graph curvature for graph classification tasks. Additionally, \cite{ni2018network} presents a dynamic graph learning framework based on Ollivier–Ricci curvature.

\section{Curvature-aware graph neural network}
The core idea of curvature-aware Graph Neural Networks is to treat curvature features as additional inputs, helping the GNN better capture the relationships and structural characteristics between nodes. 
In this section, we give the general curvature graph neural network framework.

\subsection{Message Passing}
First, we give the feature aggregation operation in a Graph Convolutional Network (GCN). Let the initial graph data be represented as $G=(N,E,c)$, where $N$ denotes the set of nodes, $E$ denotes the edge set, and $c$ is the weight matrix representing the edge weights, also known as the adjacency matrix. The initial feature matrix of the nodes is denoted as $X^0$. If the initial graph is unweighted, $c$ is represented as a binary matrix with values of 0 or 1. The message passing operation between two consecutive convolutional layers in a GNN is as follows:
\begin{equation}
    X^{i+1} = f(D^{-1}\hat{c}X^iT^i)
\end{equation}
where, $X^i$ denotes the input feature matrix of the $i$-th convolutional layer and $X^{i+1}$ is the correspnding output feature matrix, also the input of the $(i+1)$-th layer. $D$ is a diagonal matrix, $\hat{c}=c+I, D_{ii}=\sum_{j}\hat{c}_{ij}$. $T^i$ denotes the learning parameters, and $f$ is the nonlinear activity function.

\subsection{Curvature Graph Neural Network}
Edge curvature reflects the significance of edge weights, and by integrating curvature into the feature aggregation process, it helps adjust aggregation weights, thereby reducing the interference of irrelevant nodes on information propagation. When edge curvature is incorporated into the graph convolution operation, the corresponding information aggregation process is as follows:
\begin{eqnarray}
     X^{i+1} &= f(\tau_i X^i T^i),\\
    \tau_i &= g(k)\cdot D^{-1}\hat{c}.
\end{eqnarray}
$g(k)$ denotes the edge weight penalty term formed by the graph curvature $k$. Currently, most graph representation learning methods use Ollivier-Ricci curvature, but its high computational complexity is unavoidable.

\section{Effective Resistance Curvature}
In this section, we mainly introduce the effective resistance curvature as well as its estimation method with low-complexity.

\subsection{Effective Resistance}
In physics, effective resistance refers to the "minimize" resistance between two points in a circuit or network, representing the resistance encountered when current flows from one node to another~\cite{Fouzul2019}. It accounts for all resistive elements in the circuit and their interconnections, providing a comprehensive description of the current flow behavior. By aggregating the resistances across different paths, effective resistance offers a global measure of resistance that effectively captures the distribution of current and energy dissipation across the network.

In this paper, we draw an analogy between the process of information propagation on a graph and the flow of current in a network. Effective resistance is used as a metric to quantify the ease or difficulty of information flow within the network. Typically, lower effective resistance corresponds to smoother information propagation between nodes, leading to more efficient information transfer. Conversely, higher effective resistance indicates greater resistance to information transmission. Our goal is to analyze effective resistance to identify key nodes and crucial paths in the network. Nodes with lower effective resistance tend to act as center points for information flow, enhancing the efficiency of transfer, while nodes with higher effective resistance may act as bottlenecks that hinder propagation~\cite{Arpita2006}.

In electrical circuit analysis, the Laplace matrix is used to represent the nodes and connections in a circuit, facilitating the analysis of current and voltage distribution. Similarly, in our approach, we use the Laplace matrix to define the resistance matrix on a graph and compute the graph's curvature based on this matrix.

First, we give the Laplace matrix $L$ of the adjacency matrix $c$:
\begin{align}
\begin{split}
L_{ij}=\left\{
\begin{array}{lr}
-c_{ij}, & if \quad i\sim j \\
a_i, & if\quad i=j \\
0, & otherwise
\end{array}
\right.
\end{split}
\end{align}
where, $a_i=\sum_{j\sim i}c_{ij}$ denotes the weighted degree of node $i$.

The edge resistance on a graph is defined using the Moore-Penrose pseudoinverse $L^{\dagger}$ of the Laplace $L$. Correspondingly, the effective resistance between adjacent nodes is defined as follows:
\begin{equation}
    w_{ij} = (e_i-e_j)^T L^{\dagger} (e_i-e_j),
\end{equation}
where $e_i$ denotes the unit vector.

The matrix formed by the effective resistances $\{w_{ij}\}$ of all edges is referred to as the effective resistance matrix, denoted as $\Omega$. The eigenvalues and eigenvectors of the Laplace matrix provide insights into the geometric and topological structure of the graph. 
If the Laplace matrix is semi-definite, we introduce a small perturbation matrix $\epsilon\cdot I$, $\epsilon>0$ to ensure numerical stability,
\begin{equation}
    \bar{L} = L+\epsilon\cdot I
\end{equation}
This ensures that $\bar{L}$ is a positive definite symmetric matrix, allowing for direct inversion of $\bar{L}$. In the theoretical analysis section\ref{V}, we will prove that the following result hold:
\begin{equation}
    \Omega \approx \bar{\Omega}.
\end{equation}
where, $\bar{\Omega}_{ij} = (e_i-e_j)^T \bar{L}^{-1} (e_i-e_j)$. Here, it can reduce the computation complexity of the Moore-Penrose pseudoinverse $L^{\dagger}$.

\subsection{Effective Resistance Curvature}
In a network, the effective resistance between nodes is influenced not only by their direct connections but also by the overall structure of the network. Effective resistance curvature quantifies these structural variations, helping us understand the importance of nodes within the network and their role in information propagation.

We adopt the definition in~\cite{Karel2022} that introduces the concept of effective resistance curvature for nodes and edges in a graph, based on effective resistance. For weighted graphs, we introduce the concept of relative resistance $c_{ij}w_{ij}$, which quantifies the importance of an edge in maintaining overall network connectivity.

In the following, we give the definition of effective resistance curvature:
\begin{equation}
    p_i \doteq 1-\frac{1}{2}\sum_{j\sim i} c_{ij}w_{ij}.
\end{equation}
Here, $p_i$ represents the node resistance curvature, $i$ denotes the $i$-th node, and $c_{ij}w_{ij}$ represents the relative resistance between adjacent nodes $i$ and $j$. 
\begin{equation}
    k_{ij} \doteq \frac{2(p_i+p_j)}{w_{ij}}.
\end{equation}
Here, $k_{ij}$ denotes the edge resistance curvature of $i$ and $j$. 

Similar to the geometric properties of Ollivier-Ricci curvature, effective resistance curvature is designed to capture the connectivity strength within the neighborhoods of adjacent nodes. 
In the next section, we will theoretically analyze the relationship between effective resistance curvature and Ollivier-Ricci curvature.

\section{Theoretical Analysis}
\label{V}

In this section, we will analyze the low computational complexity of effective resistance curvature, demonstrate the rationality of approximate solutions for the effective resistance matrix (Proposition \ref{p1}-\ref{p3}) , and finally prove the relative consistency between effective resistance curvature and Ollivier-Ricci curvature (Proposition \ref{p4}-\ref{p7}).

\subsection{Computational Complexity Analysis}
Assume an initial graph $G=(N,E,c)$, where $N$ represents the set of nodes, $n$ denotes the number of nodes, $E$ denotes the set of edges, and $c$ is the adjacency matrix. The primary computational cost of calculating Ollivier-Ricci curvature lies in determining the optimal transport distance between the local neighborhood sets of adjacent nodes. As noted in the literature~\cite{Pal2017}, the time complexity of Ollivier-Ricci curvature is approximately $O(n^4 \log^2 n)$. For effective resistance curvature, the computational complexity is primarily determined by calculating the effective resistance matrix. This requires computing the Moore-Penrose pseudoinverse of the Laplace matrix. The Laplace matrix, derived from the adjacency matrix, is semi-positive definite and symmetric, meaning all its eigenvalues are non-negative. If the Laplace matrix is positive definite, its inverse can be directly computed. For semi-positive definite Laplace matrices, a small positive perturbation can be added to the diagonal entries to ensure positive definiteness, enabling the calculation of the matrix inverse for the effective resistance. For a general $n\times n$ matrix, the time complexity of matrix inversion is $O(n^3)$. For sparse matrices, assuming the maximum degree of a node in the graph is $m$, the time complexity of inverting the Laplace matrix is approximately $O(mn^2)$. Compared to Ollivier-Ricci curvature, the computation of effective resistance curvature is significantly more efficient, with its complexity reduced by nearly two orders of magnitude, making it more suitable for large-scale graph applications.

\subsection{Low Computation of Effective Resistance}

As previously analyzed, computing the inverse of the Laplace matrix is more efficient than computing its Moore-Penrose pseudoinverse. In this section, we focus on analyzing the approximation of the effective resistance matrix obtained from the perturbed Laplace matrix to that derived from the original Laplace matrix.

\begin{proposition}\label{p1}
    Let $G$ be an undirected graph with Laplacian matrix $L=D-A$. Then:
    \begin{enumerate}
	\item The vector $\mathbf{1} \in \mathbb{R}^n$ satisfies $L\mathbf{1}=0$, i.e., it is an eigenvector of $L$ corresponding to the eigenvalue zero.
	\item If $G$ is connected, then the all-ones vector is the unique (up to scaling) eigenvector associated with the zero eigenvalue.
\end{enumerate}
\end{proposition}

\begin{proof}
We prove each part separately.

\paragraph{ \( L \mathbf{1} = 0 \)}  
By the definition of the Laplacian matrix,
\[
L = D - A.
\]
Then,
\[
L \mathbf{1} = (D - A)\mathbf{1} = D \mathbf{1} - A \mathbf{1}.
\]
Note that \( D \mathbf{1} \) yields a vector of node degrees, and so does \( A \mathbf{1} \), since it sums the adjacency entries across each row. Therefore,
\[
D \mathbf{1} = A \mathbf{1} \quad \Rightarrow \quad L \mathbf{1} = \mathbf{0}.
\]
Hence, \( \mathbf{1} \) is in the null space of \( L \), and is thus an eigenvector associated with eigenvalue 0.

\paragraph{ Zero eigenvalue is simple if \( G \) is connected}  
Suppose \( L \mathbf{x} = 0 \) for some \( \mathbf{x} \in \mathbb{R}^n \). Consider the quadratic form associated with \( L \):
\[
\mathbf{x}^\top L \mathbf{x} = \sum_{(i,j) \in E} (x_i - x_j)^2.
\]
If \( L \mathbf{x} = 0 \), then \( \mathbf{x}^\top L \mathbf{x} = 0 \), which implies:
\[
(x_i - x_j)^2 = 0 \quad \text{for all } (i,j) \in E \Rightarrow x_i = x_j.
\]
Since \( G \) is connected, this implies that \( x_i = c \) for all \( i \), i.e.,
\[
\mathbf{x} = c \cdot \mathbf{1}.
\]
Thus, the null space of \( L \) is one-dimensional, and \( \mathrm{null}(L) = \mathrm{span}(\mathbf{1}) \), meaning that the zero eigenvalue is simple.

\end{proof}

\begin{proposition}\label{p2}
    Let $L\in \mathbb{R}^{n\times n}$ be the Laplacian matrix of a connected undirected graph (i.e., $L$ is symmetric, positive semidefinite, and has a single zero eigenvalue with eigenvector 1), and let $L^{\dagger}$ denotes its Moore-Penrose pseudoinverse. For any diagonal perturbation matrix $\Delta=\epsilon\cdot I$ with $\epsilon>0$ and $\|\Delta\|\rightarrow 0$, the inverse of the perturbed matrix $(L+\Delta)^{-1}$ does not converge to $L^{\dagger}$. That is:
    \begin{equation*}
    \lim_{\|\Delta\| \to 0} (L + \Delta)^{-1} \neq L^{\dagger}.
    \end{equation*}
\end{proposition}

\begin{proof}
    Since $L$ is a connected graph Laplacian, it admits the spectral decomposition:
   \begin{equation}\label{eq:L-decomp}
L = \sum_{i=1}^n \lambda_i \bm{v}_i \bm{v}_i^T,
\end{equation}
where $\lambda_1 \geq \lambda_2 \geq \cdots > \lambda_n =0$ are eigenvalues with corresponding orthonormal eigenvectors $\{\bm{v}_i\}_{i=1}^n$, and $\bm{v}_n = \1/\sqrt{n}$. The Pseudo-inverse is:
\begin{equation}\label{eq:pseudo}
L^\dagger = \sum_{i=1}^{n-1} \lambda_i^{-1} \bm{v}_i \bm{v}_i^T.
\end{equation}
The perturbed matrix $\bar{L} = L + \epsilon \cdot I$ has eigenvalues shifted by $\epsilon$:
\begin{equation}
\bar{L} = \sum_{i=1}^n (\lambda_i + \epsilon) \bm{v}_i \bm{v}_i^T,
\end{equation}
with its inverse given by:
\begin{equation}\label{eq:perturbed-inv}
\bar{L}^{-1} = \sum_{i=1}^n (\lambda_i + \epsilon)^{-1} \bm{v}_i \bm{v}_i^T = \frac{1}{\epsilon} \bm{v}_n \bm{v}_n^T + \sum_{i=1}^{n-1} (\lambda_i + \epsilon)^{-1} \bm{v}_i \bm{v}_i^T.
\end{equation}
Taking $\epsilon \to 0^+$ in \eqref{eq:perturbed-inv}:
\begin{align}
\lim_{\epsilon \to 0^+} \bar{L}^{-1} &= \lim_{\epsilon \to 0^+} \left( \frac{1}{\epsilon} \bm{v}_n \bm{v}_n^T \right) + \sum_{i=1}^{n-1} \lambda_i^{-1} \bm{v}_i \bm{v}_i^T \notag \\
&= \lim_{\epsilon \to 0^+} \left( \frac{1}{n\epsilon} \1\1^T \right) + L^\dagger.
\end{align}
The first term diverges while the second term converges to $L^\dagger$.
For any consistent matrix norm $\|\cdot\|$:
\begin{equation}
\left\| \bar{L}^{-1} - L^\dagger \right\| \geq \left\| \frac{1}{n\epsilon} \1\1^T \right\| - \left\| \sum_{i=1}^{n-1} \left( \frac{1}{\lambda_i + \epsilon} - \frac{1}{\lambda_i} \right) \bm{v}_i \bm{v}_i^T \right\| \to \infty.
\end{equation}
Thus, $\bar{L}^{-1}$ cannot converge to $L^\dagger$ in any norm as $\epsilon \to 0^+$.
\end{proof}

\begin{proposition}\label{p3}[Effective Resistance Stability]
Let $L \in \mathbb{R}^{n \times n}$ be the Laplacian matrix of a connected undirected graph. For sufficiently small $\epsilon > 0$, the effective resistance matrix $\bar{\Omega}$ computed from the perturbed inverse $(L + \epsilon \cdot I)^{-1}$ converges to the Pseudo-inverse based effective resistance matrix $\Omega$ as $\epsilon \to 0^+$, despite the divergence of $(L + \epsilon \cdot I)^{-1}$ in the nullspace of $L$. 
\begin{equation}
    \lim_{\epsilon \to 0^+} \bar{\Omega} \rightarrow \Omega.
\end{equation}
Specifically, the perturbation term corresponding to the zero eigenvalue is suppressed in $\bar{\Omega}$.
\end{proposition}

\begin{proof}
Assume the Laplace matrix corresponding to an undirected graph is denoted as $L$. First, perform eigenvalue decomposition on the Laplace matrix, which can be expressed as:
\begin{equation*}
    L=U\lambda U^T,
\end{equation*}
where $U_i$ denotes the eigenvector corresponding to the $i-th$ eigenvalue.

Add a small perturbation $\epsilon\cdot I$ to the Laplace matrix, resulting in:
\begin{equation*}
   \bar{L} = L+\epsilon\cdot I.
\end{equation*}
Then, the inverse of the perturbed Laplace matrix can be written as:
\begin{equation*}
    \bar{L}^{-1} = U(\Lambda+\epsilon\cdot I)^{-1} U^T,
\end{equation*}
where $\Lambda = \begin{bmatrix}
\lambda_1 & 0 & \cdots & 0 \\
0 & \lambda_2 & \cdots & 0 \\
\vdots & \vdots & \ddots & \vdots \\
0 & 0 & \cdots & \lambda_n
\end{bmatrix}$, $(\lambda_1, \lambda_2, \cdots, \lambda_n)$ denote the eigenvalues of $L$, $\lambda_n=0$.

The inverse matrix of $\bar{L}$ can be further expanded as follows:
\begin{equation*}
    \bar{L}^{-1} = [u_1, u_2,\cdots, u_n](\Lambda+\epsilon\cdot I)^{-1}\begin{bmatrix}
u_1 \\
u_2 \\
\vdots \\
u_n
\end{bmatrix}.
\end{equation*}
\begin{equation*}
    \bar{L}^{-1} = [u_1, u_2,\cdots, u_n]\begin{bmatrix}
\lambda_1^{-1}u_1 \\
\lambda_2^{-1}u_2 \\
\vdots \\
\epsilon^{-1}u_n
\end{bmatrix}.
\end{equation*}
The corresponding effective resistance matrix is computed as follows:
\begin{equation*}
    \bar{\Omega}_{ij}=(e_i-e_j)^T \bar{L}^{-1}(e_i-e_j).
\end{equation*}
\begin{equation*}
    \bar{\Omega}_{ij}=(e_i-e_j)^T[u_1, u_2,\cdots, u_n]\begin{bmatrix}
\lambda_1^{-1}u_1 \\
\lambda_2^{-1}u_2 \\
\vdots \\
\epsilon^{-1}u_n
\end{bmatrix}(e_i-e_j).
\end{equation*}
Here $u_n$ denotes the eigenvector corresponding to the $0$ eigenvalue of $L$, which is the all-ones vector. $e_i$ denotes the standard basis vector. 

Then, the final term of the vector $(e_i-e_j)^T[u_1, u_2,\cdots, u_n]$ is zero, and the same to $\begin{bmatrix}
\lambda_1^{-1}u_1 \\
\lambda_2^{-1}u_2 \\
\vdots \\
\epsilon^{-1}u_n
\end{bmatrix}(e_i-e_j)$.
\end{proof}

This stability property explains why effective resistance calculations remain numerically tractable despite the ill-posedness of $(L + \epsilon \cdot I)^{-1}$. The proof explicitly shows how the $\mathbf{1}$-orthogonality of voltage difference vectors annihilates the divergent term.

\subsection{The Relationship between Effective Resistance Curvature and Ollivier-Ricci Curvature}
\label{relation_between_O_R}
In this section, we primarily analyze the consistency between effective resistance curvature and Ollivier-Ricci curvature in graph network and data analysis.
\begin{proposition}
\label{p4}
The node and edge resistance curvature are equal to:
\begin{eqnarray}
    p_i &=& \lim_{t\rightarrow 0} (1-\frac{1}{4t}\mathbb{E}(w_{N_tM_t})),\\
    k_{ij} &=& \lim_{t\rightarrow 0} (1-\frac{\mathbb{E}(w_{N_tM_t})}{w_{ij}}),
\end{eqnarray}
    with $N_t \sim \rho_{t,i}, M_t\sim \rho_{t,j}$. Where, $\mathbb{E}(w_{N_tM_t})$ denotes the average distance between two random node sets $N_t, M_t$, defined as below:
\begin{equation}
    \mathbb{E}(w_{N_tM_t})\doteq \sum_{i,j \in \mathcal{N}}\rho_{t,i}\rho_{t,j}w_{ij}.   
\end{equation}
\end{proposition}
\begin{proof}
We follow the derivation process in~\cite{Karel2022}. The random distribution of nodes is written in the matrix form and defined as: $\rho_{t,i}=\exp{(-Qt)}e_i$, $\rho_{t,i}$ denotes the probability density function of the neighborhood nodes distribution of node $i$. Then we obtain:
\begin{equation}
    \mathbb{E}(w_{N_tM_t})=e_i^T \exp{(-Qt)\Omega \exp{(-Qt)}e_j}
\end{equation}
Take the Taylor expansion to $\exp{(-Qt)}$, when $t\rightarrow 0$, $\exp{(-Qt)}\approx I-Qt$.
\begin{align}
    \exp{(-Qt)}\Omega \exp{(-Qt)} &= (I-Qt)\Omega (I-Qt)\nonumber\\
   & = \Omega-\Omega Qt-Q\Omega t + Q\Omega Qt^2 
\end{align}
According to the definition of effective resistance curvature, we derive the following equations:
\begin{align}
    Q\Omega &=-2I+2up^T, \\
    \Omega Q &= -2I+2pu^T.
\end{align}
Where $p$ denotes a vector formed by the whole nodes resistance curvature, $u=(1, 1,\cdots, 1)^T$.
\begin{align}
   & \Omega-\Omega Qt-Q\Omega t+Q\Omega Qt^2 \nonumber\\
 &   = \Omega+4It -(2up^T+2pu^T)t+Q\Omega Qt^2
\end{align}
\begin{align}
    &e_i^T(\Omega-\Omega Qt-Q\Omega t + Q\Omega Qt^2)e_i \nonumber\\
    & = e_i^T (\Omega+4It -(up^T+2pu^T)t + Q \Omega Qt^2)e_i \nonumber\\
    & = 4t-4p_i t +e_i^T Q\Omega Q e_i t^2.
\end{align}
\begin{align}
    & e_i^T(\Omega-\Omega Qt-Q\Omega t + Q\Omega Qt^2)e_j \nonumber\\
    & = e_i^T (\Omega+4It -(up^T+2pu^T)t + Q \Omega Qt^2)e_j \nonumber\\
    & = w_{ij}-2(p_i+P_j)t+e_i^TQ\Omega Q e_j t^2.
\end{align}
According to the above derivation process, we derive the proposition 1.
\begin{align}
\lim_{t\rightarrow 0} (1-\frac{1}{4t}\mathbb{E}(w_{N_tM_t})) 
&= \lim_{t\rightarrow 0}(1-\frac{4t-tp_it+e_i^T Q\Omega Q e_i t^2}{4t})\nonumber\\
& = \lim_{t\rightarrow 0}(1-1+p_i-\frac{e_i^T Q \Omega e_it}{4}) \nonumber\\
&= p_i.
\end{align}
\begin{align}
&\lim_{t\rightarrow 0}\frac{1}{t}(1-\frac{\mathbb{E}(w_{N_tM_t})}{w_ij}) \\
&=\lim_{t\rightarrow 0}\frac{1}{t}(1-\frac{w_ij-2(p_i+p_j)+e_i^T Q \Omega Qe_j t^2}{w_ij})\nonumber\\
&=\lim_{t\rightarrow 0}\frac{1}{t}(1-1+\frac{2(p_i+p_j)t}{w_{ij}}-\frac{e_i^T Q\Omega Q e_j t^2}{w_{ij}})\nonumber\\
&=\lim_{t\rightarrow 0}(\frac{2(p_i+p_j)}{w_ij}-\frac{e_i^T Q\Omega Q e_jt}{w_{ij}})\nonumber\\
&= k_{ij}
\end{align}
\end{proof}

Based on Proposition 1, we analyze the relationship between edge resistance curvature and Ollivier-Ricci curvature. First, we give the formula of Ollivier-Ricci curvature, show below:
\begin{equation}
    k_{ij}^{(OR)} = \lim_{t\rightarrow 0}\frac{1}{t} (1-\frac{W_1(\mu_{t,i},\mu_{t,j})}{d_{ij}})
\end{equation}
$W_1(\mu_{t,i},\mu_{t,j})$ represents the optimal transportation distance between the neighborhood sets of nodes $i$ and $j$.

Based on the computation formula for Ollivier-Ricci curvature, its distinction from the calculation of effective resistance curvature is that Ollivier-Ricci curvature is designed to calculate the optimal transport distance between the local neighborhoods of two nodes on an edge. In contrast, effective resistance curvature focuses on computing the average resistance transport distance between node sets. Effective resistance can be used to approximate the shortest path between two nodes. Here, if we regard the effective resistance on an edge as the shortest path between two nodes, under this distribution ,the optimal transport distance between the two nodes should be less than the average transport distance. Thus, it can be observed that:
\begin{equation}
    \mathbb{E}(w_{N_tM_t}) \geq W_1(\mu_{t,i},\mu_{t,j}).
\end{equation}

From this perspective, the effective resistance curvature is lower than Ollivier-Ricci curvature: $k_{ij} \leq k_{ij}^{(OR)}$. However, this holds true only in the limit when the local neighborhoods are very small.
\begin{equation}
    \mathbb{E}(w_{N_tM_t}) = w_{ij}-2(p_i+p_j)+e_i^T Q\Omega Q e_j t^2.
\end{equation}

According to the definition of node resistance curvature, the above formula can be transformed to the following form:
\begin{equation}
    \mathbb{E}(w_{N_tM_t}) = w_{ij}-4+\sum_{k\sim i}w_{ik}c_{ik}+\sum_{l\sim j/{l\sim i}}w_{jl}c_{jl}.
\end{equation}

This can be viewed as the sum of the relative resistance distances of node pairs with respect to their local neighborhoods. The higher the connection degree between two nodes, the smaller the corresponding relative resistance distance will be, which is consistent with the optimal transport distance.

\begin{proposition}
\label{p5}
    Given an adjacent graph $G=(N,E,c)$, the bound of the effective resistance curvature $k_{ij}$ is as follows:
    \begin{equation}
        \frac{4-d_i-d_j}{w_{ij}}\leq k_{ij}\leq \frac{2}{w_{ij}}.
    \end{equation}
\end{proposition}
\begin{proof}
    For each edge $(i,j)$, the effective resistance curvature is shown below:
    \begin{equation}
        k_{ij}\doteq \frac{2(p_i+p_j)}{w_{ij}}.
    \end{equation}

    Moreover, according to the definition of node resistance curvature, $p_i\leq \frac{1}{2}$, we can derive that:
    \begin{equation}
        k_{ij}\leq \frac{2}{w_{ij}}.
    \end{equation}
    Furthermore, 
    \begin{equation}
        k_{ij}\doteq \frac{4-\sum_{i\sim l}c_{il}w_{il}-\sum_{j\sim h}c_{jh}w_{jh}}{w_{ij}}.
    \end{equation}

    Suppose the initial graph is unweighted, the corresponding adjacent matrix is a $(0,1)$ matrix, with $c_{ij}=1$, then the above formula can be transformed to the following inequation:
    \begin{equation}
        k_{ij}\geq \frac{4-d_i-d_j}{w_{ij}}.
    \end{equation}

    Assuming $w_{\max}$ represents the maximum effective resistance on the edges of the graph, and $w_{\min}$ denotes the minimum effective resistance, then the above inequality can also be expressed as:
    \begin{equation}
        k_{ij}\geq \frac{4-(d_i+d_j)}{w_{\min}},
    \end{equation}
    where $d_i$ denotes the degree of node $i$.

    To maintain consistency with the bounds of Ollivier-Ricci curvature, we normalize the effective resistance curvature on the graph by multiplying all effective resistance curvatures by $\frac{w_{\min}}{2}$. Then the bounds of the normalized curvature are expressed as follows:
    \begin{equation}
        \frac{4-(d_i+d_j)}{2} \leq \bar{k}_{ij} \leq 1.
    \end{equation}
\end{proof}

For unweighted graph, the resistance of any edge is 1. The minimum curvature of an edge occurs when there is no parallel path between the two nodes, at which node the effective resistance reaches its maximum value of 1. Therefore, in the case of unweighted graphs, we can accurately establish the bounds of the normalized effective resistance curvature, which  are essentially consistent with the Ollivier-Ricci curvature.

\begin{proposition}\label{p6}
    The pairwise edge effective resistance curvature strictly increases when edges are added or weights are increased.
\end{proposition}
\begin{proof}
    We follow the Theorem 2.6 and 2.7 of~\cite{ellens2011}. The pairwise effective resistance does not increase when edges are added or weights are increased.

    The effective graph resistance strictly decreases when edges are added or weights are increased.
\end{proof}

We interpret the effective resistance curvature of a node as a measure of its connectivity within its local neighborhood. Consider a subgraph consisting of a node and its direct neighbors. As edges are added among these nodes, the total effective resistance—also referred to as the local effective graph resistance—decreases, which in turn leads to an increase in the effective resistance curvature. This confirms that higher connectivity density in the local neighborhood results in greater effective resistance curvature of the corresponding node or edge. Notably, this behavior is consistent with that of the Ollivier–Ricci curvature. Theoretically, it can also be shown that in the case of Ollivier–Ricci curvature, adding edges within a local region similarly leads to an increase in the curvature of those edges.

\begin{proposition}\label{p7}
    The sum of the optimal transport distance and the local relative resistance distance between the neighborhoods of two nodes remains consistent. Specifically, as the optimal transport distance decreases, the sum of the relative resistance distances also decreases, and vice versa.
\end{proposition}
\begin{proof}
    Given any pair of adjacent nodes $(v_i,v_j)$ in a graph, let the local neighborhood set of node $v_i$ be denoted as $N_i$, and the local neighborhood set of node $v_i$ as $N_j$. If the local neighborhood sets of the two nodes remain fixed, increasing the number of common neighbors between them i.e., adding edges that connect nodes to the opposite neighborhood set, will reduce the optimal transport distance between the node pair, assuming edge weights remain unchanged, as per the definition of optimal transport distance. Additionally, according to Proposition \ref{p6}, when edges between nodes are added, the effective resistance between them decreases, while the corresponding effective resistance curvature increases. This change aligns with the reduction in optimal transport distance.
\end{proof}
\section{Experiments}

This section evaluates the following three pivotal scientific questions through systematic experimentation to demonstrate the advantage of effective resistance curvature:
\begin{enumerate}
\item \textbf{Q1: }Can effective resistance curvature-based graph representation learning yield performance comparable to Ollivier-Ricci curvature-based approaches under identical experimental setting?
\item \textbf{Q2: }For graphs of comparable scale, what is the computational speedup of effective resistance curvature relative to Ollivier-Ricci curvature?
\item \textbf{Q3: }What are the consistency and divergence between effective resistance curvature and Ollivier-Ricci curvature in empirical analysis?
\end{enumerate}

Based on the aforementioned issues, our experiments can be divided into three parts. 
\begin{enumerate}
    \item \textbf{Evaluation of Representation Learning:} We evaluate the integration of both Ollivier-Ricci curvature and effective resistance curvature into graph representation learning. Their effectiveness is assessed in the context of node classification and graph classification tasks, using both real-world and synthetic datasets. 
    \item \textbf{Efficiency Benchmarking:} We systematically evaluated the computational efficiency of both Ollivier-Ricci and effective resistance curvature methods under multiple configurations, utilizing a mix of real-world and synthetic datasets for a comprehensive comparison. 
    \item \textbf{Empirical Comparison:} We investigated the consistency and divergence between effective resistance curvature and Ollivier-Ricci curvature on real-world datasets, analyzing their distributions and characteristic geometric properties. 
\end{enumerate}

We have open-sourced the implementation for computing effective resistance curvature and packaged it on PyPI (install via \textit{pip install ResistanceCurvature}). Source code is available at	
\url{https://github.com/cqfei/resistance\_curvature}. The complete code for this paper is published on GitHub:	
\url{https://github.com/cqfei/Efficient-Curvature-aware-Graph-Network}.

\begin{table*}[!hbpt]
\centering
\caption{Statistics for graph node classification dataset}
\label{tab:node_classification_datasets}
\begin{tabular}{lccccccc}
\hline
\textbf{Dataset} & \textbf{Cora} & \textbf{Citeseer} & \textbf{PubMed} & \textbf{Amazon Computers} & \textbf{Amazon Photo} & \textbf{Coauthor CS} & \textbf{Coauthor Physics} \\
\hline
nodes & 2,708 & 3,327 & 19,717 & 13,752 & 7,650 & 18,333 & 34,493 \\
edges & 5,429 & 4,732 & 44,338 & 245,861 & 119,081 & 81,894 & 247,962 \\
classes & 7 & 6 & 3 & 10 & 8 & 15 & 5 \\
features & 1,433 & 3,703 & 500 & 767 & 745 & 6,805 & 8,415 \\
density & $1.44 \times 10^{-3}$ & $0.82\times 10^{-3}$ & $0.23\times 10^{-3}$ & $2.60\times 10^{-3}$ & $4.07\times 10^{-3}$ & $0.49\times 10^{-3}$ & $0.42\times 10^{-3}$ \\
in/out & 2.79 & 10.11 & 4.05 & 3.48 & 4.78 & 4.21 & 13.49 \\
\hline
\end{tabular}
\end{table*}

\begin{table*}[!hbpt]
    \centering
    \caption{Statistics for graph classification dataset}
    \label{tab:graph_classification_datasets}
    \begin{tabular}{lcccccc}
        \hline
        dataset & Category & Nodes(AVG) & Edges(AVG) & \#graph number & \#classes & mean density\\
        \hline
        ENZYMES & BIOINFORMATICS & 32.6 & 62.1 & 600 & 6 & 0.121\\
        MUTAG & BIOINFORMATICS & 17.9 & 19.8 & 188 & 2 & 0.131\\
        PROTEINS & BIOINFORMATICS & 39.1 & 72.8 & 1113 & 2 & 0.098\\
        D\&D & BIOINFORMATICS & 284.3 & 715.7 & 1178 & 2 & 0.018\\
        IMDBMULTI & SOCIAL NETWORK & 13.0 & 65.9 & 1500 & 3 & 0.845\\
        COLLAB & SOCIAL NETWORK & 74.5 & 2457.8 & 5000 & 3 & 0.898\\
        IMDBBINARY & SOCIAL NETWORK & 19.8 & 96.5 & 1000 & 2 & 0.519\\
        \hline
    \end{tabular}
\end{table*}

\subsection{Experimental Settings}
\subsubsection{Experimental Datasets}

\textbf{Real-world datasets for graph node representation learning.} We utilize 7 real-world benchmark dataset including Cora\footnote{https://github.com/shchur/gnn-benchmark/tree/master/data/planetoid}, Citeseer\footnotemark[1], PubMed\footnotemark[1], Amazon Computers\footnote{https://github.com/shchur/gnn-benchmark/tree/master/data/npz}, Amazon Photo\footnotemark[2], Coauthor CS\footnotemark[2] and Coauthor Physics\footnotemark[2]. The statistical details of these datasets are summarized in  Table~\ref{tab:node_classification_datasets}. 

\textbf{Synthetic datasets for graph node representation learning.}
We conduct synthetic experiments using several well-established graph theoretical models: Watts–Strogatz model~\cite{WattsStrogatz1998}, Newman–Watts model~\cite{NewmanWatts1999} and Kleinberg’s navigable small world graphs~\cite{Kleinberg2000}. The Watts-Strogatz model generates networks by randomly rewiring edges of a ring graph, while the Newman-Watts model adds random edges to the ring structure. Kleinberg's model also introduces random edges, but with probabilities inversely proportional to the shortest-path distances on the ring. Additional details are provided in the
Appendix~\ref{a_graphmodel}.

\textbf{Real-world datasets for graph classification (graph pooling).}
We take graph pooling in graph classification tasks as the experimental subject to validate the effectiveness of graph curvature in such tasks, with experiments conducted on 5 benchmark
datasets including three bioinformatics datasets: ENZYMES~\cite{Schomburg2004}, PROTEINS~\cite{Karsten2005}, D\&D~\cite{D&D2003}, and two social networks: IMDB-B~\cite{Pinar2015a}, IMDB-M~\cite{Pinar2015b}. The statistical properties of these datasets are reported in Table~\ref{tab:graph_classification_datasets}.

\subsubsection{Dataset Division}
\textbf{Real-world datasets for graph node classification}. We follow the dataset splitting strategy of~\cite{Fei&Li2025}. For Cora, Citeseer, and PubMed datasets, we use the public splits provided with the datasets. For Amazon Computers, Amazon Photo, Coauthor CS, and Coauthor Physics datasets, which lack predefined splits, we employ random partitioning: 10\% of nodes for training, 10\% for validation, and the remaining 80\% for testing.

\textbf{Synthetic datasets for graph node classification}. We generate synthetic datasets following~\cite{li2022curvature}. Each dataset contains 1,000 nodes equally divided into 5 classes. We randomly select 20 nodes per class for training, 300 nodes for validation, and the remaining 600 nodes for testing. Each node is assigned a randomly generated 20-dimensional feature vector. Additionally, we generate 10-dimensional uninformative features for each node as model inputs.

\subsubsection{Training Details}
\textbf{Real-world dataset for graph node classification}. We adopt the linear evaluation scheme as in DGI~\cite{Velickovic2019}, where models are first trained in an unsupervised manner. The resulting node embeddings are then used to train and test an L2-regularized logistic regression classifier. We report the average performance over 10 independent runs for each dataset.

\textbf{Real-world dataset for graph pooling}. Following the DGCNN~\cite{ZhangDGCNN2018}, we perform 10-fold cross-validation (9 folds for training, 1 fold for testing) with one training fold reserved for hyperparameter tuning. Experiments are repeated 10 times, and we report the mean classification accuracies.
\subsubsection{Baseline Methods}
\textbf{Baselines for graph node classification.} 
We compare against four state-of-the-art models on both real-world and synthetic datasets: GCN~\cite{GCN2016}, GAT with concatenation~\cite{GAT2017}, CurvGN~\cite{ye2019}, and CGNN~\cite{li2022curvature}.

GCN explicitly utilizes node degrees to compute feature weights, while GAT, CurvGN, and CGNN implicitly learn feature weights through data-driven mechanisms. GAT employs self-attention mechanisms operating on hidden layer features. CurvGN transforms Ricci curvature into multi-channel feature weights via MLP. CGNN leverages graph curvature properties through specialized modules including Negative Curvature Transformation Module and Curvature Normalization Module.

For CurvGN, we use CurvGN-n (the multi-valued variant) as our baseline. For CGNN, we evaluate all six variants: CGNN\_Linear\_$1^{st}$, CGNN\_Linear\_$2^{nd}$, CGNN\_Linear\_Sym, CGNN\_Exp\_$1^{st}$, CGNN\_Exp\_$2^{nd}$, and CGNN\_Exp\_Sym. The naming convention indicates the transformation type (Linear/Exponential) and normalization scheme ($1^{st}$-hop/$2^{nd}$-hop/symmetric).

\textbf{Baselines for graph pooling.} We compare our graph pooling method(CurPool-*) against state-of-the-art GCNN-based pooling approaches, including SUMPool~\cite{selfpool}, SORTPool~\cite{zhang2018graphclassification}, TOPKPool~\cite{gao2019}, SAGPool~\cite{selfpool}, DIFFPool~\cite{diffpool}, StructPool~\cite{structpool}, MinCutPool~\cite{spectralpool}, SEP~\cite{SEP}, HaarPool~\cite{haar}, and iPool~\cite{ipool}.

\subsubsection{Evaluation Metrics}
For graph node classification tasks (both real-world and synthetic datasets), we report the mean and standard deviation of classification accuracy on test nodes. For graph pooling tasks, we follow established practices and use mean classification accuracy as the primary evaluation metric.

\subsubsection{Computing Infrastructure}
All the models are implemented in PyTorch 1.12.0 and Python 3.9. Experiments are conducted on a server equipped with an Intel(R) Xeon(R) Gold 6248R CPU, NVIDIA RTX 3090 GPU (24GB VRAM), and 256GB RAM. The Ollivier-Ricci curvature computations are performed using a Python implementation\footnote{https://github.com/saibalmars/GraphRicciCurvature} from~\cite{ni2019community}. 

\subsection{Experiments on Graph Representation Learning}
\subsubsection{Graph Node Classification on Real-world Datasets.}
We conducted comprehensive experiments across seven real-world datasets, with detailed results presented in Table~\ref{tab:real_node_classification_result}. Seven curvature integration methods were implemented in GNN frameworks: CurvGN, CGNN\_Exp\_$1^{st}$, CGNN\_Exp\_$2^{nd}$, CGNN\_Exp\_Sym,CGNN\_Linear\_$1^{st}$,CGNN\_Linear\_$2^{nd}$,\\
CGNN\_Linear\_Sym. Each method integrated either Ollivier-Ricci curvature (denoted as **-O) or effective resistance curvature (denoted as **-R). Hyperparameter configurations are documented in the 
Appendix~\ref{a_Hyperparameter}.

The results from Table~\ref{tab:real_node_classification_result} reveal two key observations:
\begin{enumerate}
\item Curvature-enhanced methods consistently outperformed non-curvature baselines across all datasets, validating the effectiveness of curvature integration.
\item The performance gap between effective resistance curvature and Ollivier-Ricci curvature across the seven methods ranged from -0.32\% to 0.3\% (mean absolute difference: $0.16\%$), indicating statistically comparable capabilities in characterizing graph structural information.
\end{enumerate}

\begin{table*}
\centering
\caption{Comparison of Node Classification Accuracy (\%) on Real-world Datasets(Mean ±Std)}
\label{tab:real_node_classification_result}
\begin{tabular}{l *{7}{c}}
\hline
\textbf{Model} & \textbf{Cora} & \textbf{Citeseer} & \textbf{PubMed} & \textbf{Amazon Photo} & \textbf{Amazon Computers} & \textbf{Coauthor CS} & \textbf{Coauthor Physics} \\
\hline
GCN  & 81.5$\pm$1.3 & 71.9$\pm$0.9 & 77.8$\pm$2.9 & 91.2$\pm$1.2 &  82.6$\pm$2.4 & 91.1$\pm$0.5 & 92.8$\pm$1.0 \\
GAT  & 81.8$\pm$1.3 & 71.4$\pm$1.9 & 78.7$\pm$2.3 & 85.7$\pm$20.3 & 78.0$\pm$19.0 & 90.5$\pm$0.5 & 92.5$\pm$0.9 \\
\hline
CurvGN-O & 80.62$\pm$1.01 & \textbf{70.83}$\pm$0.62 & 78.22$\pm$0.30 & 91.76$\pm$0.39 & \textbf{84.03}$\pm$0.58 & 92.70$\pm$0.32 & 93.18$\pm$0.44 \\
CurvGN-R & \textbf{81.10}$\pm$1.06 & 70.68$\pm$0.93 & \textbf{79.22}$\pm$0.38 & \textbf{92.25}$\pm$0.51 & 82.99$\pm$0.67 & \textbf{92.76}$\pm$0.38 & \textbf{93.21}$\pm$0.38 \\
\hline
CGNN\_Linear\_$1^{st}$-O & 82.22$\pm$0.52 & 71.19$\pm$0.87 & \textbf{78.02}$\pm$0.41 & \textbf{91.87}$\pm$0.62 & 83.63$\pm$0.58 & \textbf{91.86}$\pm$0.32 & \textbf{93.21}$\pm$0.70 \\
CGNN\_Linear\_$1^{st}$-R & \textbf{82.89}$\pm$0.54 & \textbf{71.55}$\pm$0.67 & 77.91$\pm$0.43 & 91.44$\pm$0.54 & 83.58$\pm$0.56 & 90.95$\pm$0.36 & 92.60$\pm$0.73 \\
\hline
CGNN\_Linear\_$2^{nd}$-O & 81.97$\pm$0.47 & 71.65$\pm$0.61 & 78.20$\pm$0.47 & 91.54$\pm$0.66 & \textbf{84.52}$\pm$0.56 & \textbf{92.99}$\pm$0.35 & \textbf{93.76}$\pm$0.59 \\
CGNN\_Linear\_$2^{nd}$-R & \textbf{82.75}$\pm$0.41 & \textbf{71.53}$\pm$0.55 & \textbf{78.65}$\pm$0.56 & \textbf{91.42}$\pm$0.76 & 83.81$\pm$0.49 & 92.39$\pm$0.54 & 93.30$\pm$0.52 \\
\hline
CGNN\_Linear\_Sym-O & 82.28$\pm$0.45 & 71.25$\pm$0.71 & 78.12$\pm$0.71 & \textbf{91.89}$\pm$0.45 & 84.00$\pm$0.56 & 92.15$\pm$0.50 & \textbf{93.38}$\pm$0.52 \\
CGNN\_Linear\_Sym-R & \textbf{83.10}$\pm$0.53 & \textbf{71.72}$\pm$0.59 & \textbf{78.89}$\pm$0.43 & 91.84$\pm$0.45 & 83.79$\pm$0.51 & 92.09$\pm$0.30 & 92.88$\pm$0.70 \\
\hline
CGNN\_Exp\_$1^{st}$-O & \textbf{83.39}$\pm$0.34 & \textbf{71.17}$\pm$0.72 & \textbf{78.16}$\pm$0.42 & \textbf{91.85}$\pm$0.58 & 83.56$\pm$0.54 & \textbf{91.33}$\pm$0.53 & 92.42$\pm$0.66 \\
CGNN\_Exp\_$1^{st}$-R & 81.75$\pm$0.57 & 71.12$\pm$0.71 & 77.95$\pm$0.53 & 91.64$\pm$0.49 & \textbf{84.01}$\pm$0.50 & 91.04$\pm$0.43 & \textbf{92.43}$\pm$0.59 \\
\hline
CGNN\_Exp\_$2^{nd}$-O & 82.61$\pm$0.49 & \textbf{71.80}$\pm$0.52 & 78.57$\pm$0.56 & 91.59$\pm$0.48 & 84.30$\pm$0.48 & \textbf{92.66}$\pm$0.45 & 93.24$\pm$0.55 \\
CGNN\_Exp\_$2^{nd}$-R & \textbf{83.04$\pm$0.50} & 71.58$\pm$0.70 & \textbf{78.62}$\pm$0.37 & \textbf{91.87}$\pm$0.62 & 84.10$\pm$0.56 & 92.58$\pm$0.46 & \textbf{93.35}$\pm$0.59 \\
\hline
CGNN\_Exp\_Sym-O & \textbf{83.08}$\pm$0.51 & 71.19$\pm$0.63 & \textbf{78.92}$\pm$0.51 & \textbf{91.84}$\pm$0.46 & 83.81$\pm$0.60 & \textbf{92.41}$\pm$0.51 & \textbf{93.46}$\pm$0.67 \\
CGNN\_Exp\_Sym-R & 82.10$\pm$0.51 & \textbf{71.31}$\pm$0.58 & 78.90$\pm$0.35 & 91.87$\pm$0.47 & \textbf{84.17}$\pm$0.52 & 92.14$\pm$0.40 & 92.97$\pm$0.62 \\
\hline
AVG-O & 82.30 & 71.26 & 78.29 & \textbf{91.76} & \textbf{83.93} & \textbf{92.14} & \textbf{93.25} \\
AVG-R & \textbf{82.30} & \textbf{71.39} & \textbf{78.59} & 91.67 & 83.75 & 91.92 & 93.04 \\
\hline
\end{tabular}
\end{table*}

\subsubsection{Graph Node Classification on Synthetic Datasets}
Experiments were conducted on three synthetic graph models under high-density (**-H) and low-density (**-L) configurations (to assess the impact of topological density). Each trial was repeated 100 times for statistical robustness (hyperparameters in the 
Appendix~\ref{a_Hyperparameter}).

We conclude 3 key findings from Table~\ref{tab:node_classification_generated_results}. 
\begin{enumerate}
\item Curvature integration significantly improved performance over curvature-free baselines. This validates curvature integration as a critical enhancement for graph structural learning. 

\item Effective resistance curvature achieved performance parity with Ollivier-Ricci curvature (difference: -0.55\% to 0.13\%). This establishes effective resistance curvature as a viable alternative for graph classification tasks, particularly in resource-constrained scenarios. 
\item Graph density is a critical topological factor for GNN learning. Empirical results on diverse synthetic graphs indicate that higher graph density generally enhances GNN performance by optimizing the efficiency of the message-passing mechanism. Moreover, curvature-enhanced GNNs can further strengthen the model's capacity to represent the underlying graph geometry across varying density conditions.

\end{enumerate}

\begin{table*}[htbp]
\centering
\caption{Node Classification Accuracy (\%) on Synthetic Graphs}
\label{tab:node_classification_generated_results}
\begin{tabular}{lcccccc}
\hline
\textbf{Model} & \textbf{WS-H} & \textbf{WS-L} & \textbf{NW-H} & \textbf{NW-L} & \textbf{KL-H} & \textbf{KL-L} \\
\hline
GCN            & 41.485 ± 1.97 & 25.590 ± 1.51 & 42.250 ± 1.29 & 26.095 ± 0.91 & 56.035 ± 2.10 & 37.620 ± 2.24 \\
GAT            & 42.255 ± 2.69 & 25.045 ± 1.17 & 40.810 ± 1.81 & 27.580 ± 1.43 & 49.445 ± 2.63 & 35.465 ± 1.93 \\
\hline
CGNN-line-sy-O & \textbf{44.340 ± 1.86} & 24.675 ± 1.50 & \textbf{50.810 ± 1.26} & \textbf{33.560 ± 1.10} & 62.880 ± 2.69 & 36.990 ± 1.17 \\
CGNN-line-sy-R & 43.935 ± 2.10 & \textbf{25.350 ± 1.13} & 50.660 ± 1.35 & 32.885 ± 1.83 & \textbf{64.170 ± 2.02} & \textbf{38.330 ± 1.49} \\
\hline
CGNN-line-1-O  & \textbf{44.505 ± 1.64} & 24.655 ± 1.30 & 50.905 ± 0.93 & 32.840 ± 0.97 & 62.590 ± 1.84 & 37.810 ± 1.46 \\
CGNN-line-1-R  & 44.320 ± 1.38 & \textbf{24.780 ± 1.23} & \textbf{51.865 ± 1.42} & \textbf{33.125 ± 0.97} & \textbf{65.120 ± 2.30} & \textbf{38.765 ± 1.41} \\
\hline
CGNN-line-2-O  & 43.835 ± 1.09 & 24.595 ± 1.25 & 50.625 ± 1.57 & 33.695 ± 1.23 & 62.455 ± 1.98 & 37.690 ± 1.00 \\
CGNN-line-2-R  & \textbf{44.240 ± 1.61} & \textbf{25.200 ± 1.37} & \textbf{49.790 ± 1.80} & \textbf{33.850 ± 1.18} & \textbf{62.880 ± 2.25} & \textbf{38.045 ± 1.13} \\
\hline
CGNN-exp-sy-O  & 44.570 ± 1.19 & 24.865 ± 1.54 & 50.795 ± 1.49 & \textbf{34.130 ± 0.69} & 63.675 ± 2.29 & 37.665 ± 1.47 \\
CGNN-exp-sy-R  & \textbf{45.295 ± 1.54} & \textbf{25.035 ± 1.39} & \textbf{50.450 ± 1.28} & 33.530 ± 1.44 & \textbf{63.500 ± 2.06} & \textbf{38.835 ± 1.44} \\
\hline
CGNN-exp-1-O   & \textbf{44.640 ± 1.87} & 24.890 ± 1.29 & 51.005 ± 1.27 & 33.325 ± 1.05 & 63.000 ± 1.91 & 38.265 ± 1.21 \\
CGNN-exp-1-R   & 44.050 ± 1.18 & \textbf{25.365 ± 1.40} & \textbf{51.290 ± 1.40} & \textbf{32.990 ± 0.90} & \textbf{64.210 ± 1.54} & \textbf{38.630 ± 2.31} \\
\hline
CGNN-exp-2-O   & \textbf{44.310 ± 1.16} & \textbf{25.300 ± 1.32} & 50.310 ± 1.61 & 33.505 ± 1.66 & 63.975 ± 1.83 & 38.700 ± 1.18 \\
CGNN-exp-2-R   & 43.975 ± 2.12 & 24.345 ± 1.11 & \textbf{50.920 ± 1.11} & \textbf{34.215 ± 0.99} & \textbf{64.375 ± 1.80} & \textbf{39.125 ± 1.08} \\
\hline
CGN-N-O        & \textbf{46.020 ± 1.27} & \textbf{27.470 ± 1.28} & \textbf{48.790 ± 1.44} & \textbf{27.565 ± 1.06} & \textbf{67.845 ± 1.79} & 39.690 ± 1.68 \\
CGN-N-R        & 45.980 ± 2.15 & 26.875 ± 1.46 & 47.520 ± 1.25 & 26.605 ± 0.89 & 67.520 ± 1.47 & \textbf{39.945 ± 1.54} \\
\hline
Avg-O & \textbf{44.603} & 25.064& \textbf{50.463}& \textbf{32.660}& 63.774& 38.116\\
Avg-R & 44.542 & \textbf{25.279}& 50.356& 32.457& \textbf{64.539} & \textbf{38.811} \\
\hline
\end{tabular}
\end{table*}

\begin{figure*}[!htbp]
	\centering 
	\includegraphics[width=1.0\textwidth]{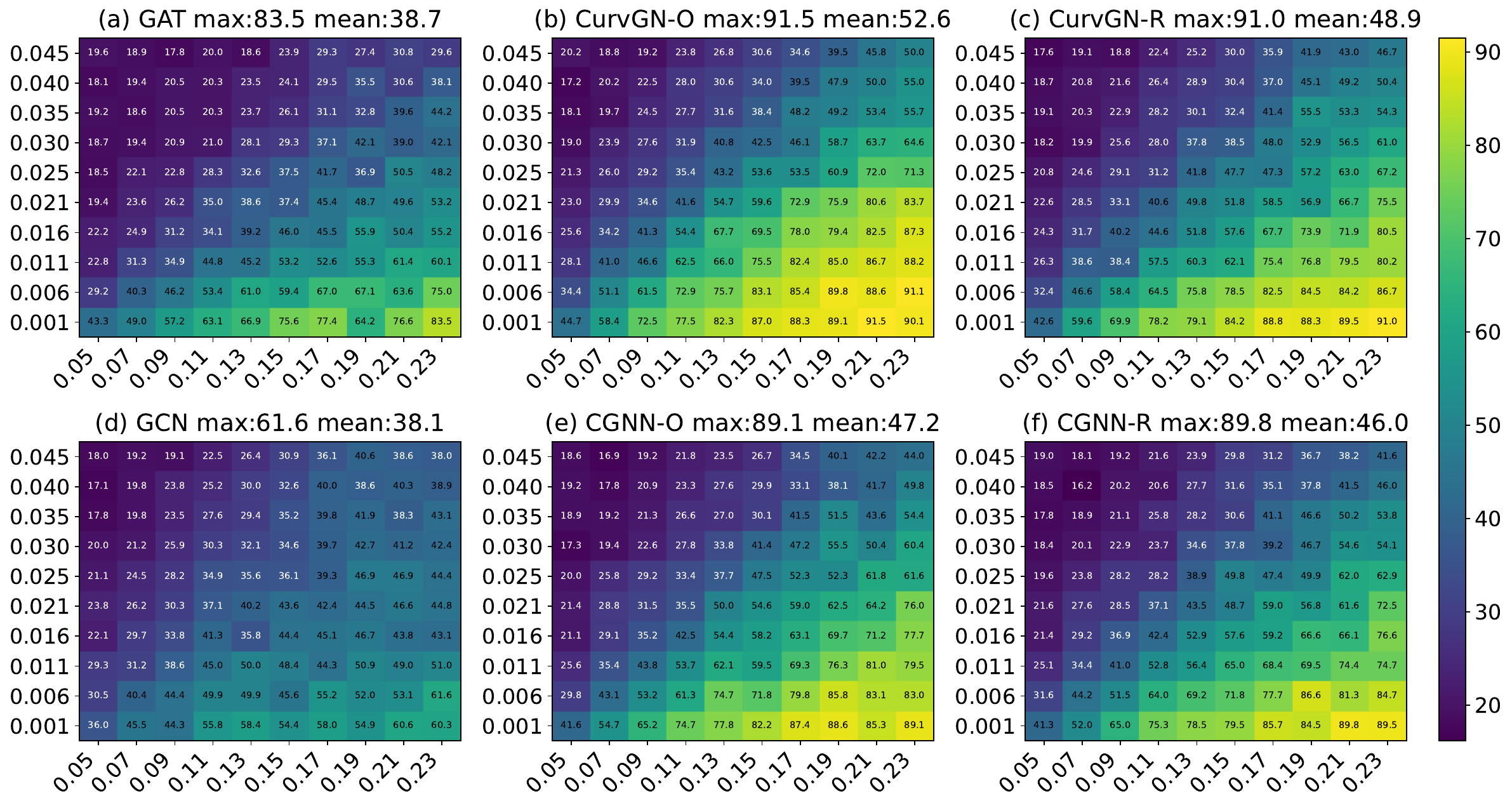} 
    \caption{Classification accuracy heatmaps on SBM-generated graphs. (a,d): GAT/GCN baselines; (b,c,e,f): CurvGN/CGNN with Ollivier-Ricci (-O) and resistance (-R) curvature. (Axes: parameter p vs. parameter q)}
	\label{fig:sbm_heatmap}
\end{figure*}

Furthermore, to evaluate how community structure influences node feature aggregation, we examined the classification accuracy of different curvature integration methods on 100 graphs generated using the Stochastic Block Model (SBM), as visualized in Fig.~\ref{fig:sbm_heatmap}. In the SBM, parameter $p$ controls the intra-community connection probability, while $q$ controls the inter-community connection probability. By systematically varying pand q, we analyze the effect of integrating different curvature types. Following~\cite{ye2019}, we set $p \in \{0.05, 0.07, \dots, 0.23\}$ and $q \in \{0.001, 0.005, \dots, 0.045\}$. Each graph configuration was evaluated over 10 independent runs, and results were averaged. Models were trained for 200 epochs with early stopping based on a validation set.

The heatmaps reveal several key findings:
\begin{itemize}
    \item Models incorporating curvature consistently outperform those without curvature.
    \item Classification accuracy generally increases from the top-left (low p, high q) to the bottom-right (high p, low q) across all heatmaps, indicating that GNNs perform better when community structure is strong—i.e., high intra-community and low inter-community connectivity. Curvature integration further amplifies this tendency.
    \item Under the CurvGN framework, Ollivier-Ricci curvature yields better performance than effective resistance curvature, whereas under the CGNN framework, the two perform comparably. This implies that only with an appropriate curvature integration method can the full potential of curvature be realized, and that current techniques—largely designed for Ollivier-Ricci curvature—may not fully harness the potential of effective resistance curvature. Future work could explore custom integration schemes that better align with the geometric properties of effective resistance curvature.
\end{itemize}

\subsubsection{Experiments for Graph Pooling}
\begin{table*}[h]
    \centering
    \caption{Graph Classification Accuracy (\%) on Benchmark Datasets}
    \label{tab:graph_pool_results}
    \begin{tabular}{lccccc}
        \hline
        dataset & ENZYMES & D\&D & PROTEINS  &  IMDB-B & IMDB-M   \\
        \hline
		SumPool    & 47.33& 78.72 & 76.26 & 51.69 & 42.76  \\
		SortPool  & 52.83& 80.60 & 76.83 & 70.00 & 46.26\\
		TopkPool    & 53.67& 81.71 &  77.47& 72.80 & 49.00 \\
        SagPool  & \textbf{64.17}& 81.03  & 78.43 & 73.40 & 51.13 \\
		DiffPool   & 60.33& 80.94& 77.74& 72.40& 50.13 \\
		StructPool  & 62.69& 82.22& 78.74& 73.50 &51.20 \\
		MinCutPool & 52.00* & 81.37  &76.52 &72.67 & 50.83\\
        SEP & - & 77.98  &76.42 &74.12 & 51.53\\
        iPool & 59.50 & 79.45  & 77.63 &- & -\\
        HaarPool & 41.67 & 81.20  & 80.40 &- & -\\
        \hline
        CurvPool-O & 57.00 & 82.90 & \textbf{79.63}   & 77.66 & \textbf{53.33}   \\
        CurvPool-R & \textbf{58.00} & \textbf{82.90} & 79.60    & \textbf{78.00} & 52.7   \\
        \hline
    \end{tabular}
\end{table*}

To validate the effectiveness of graph curvature in graph classification, we focus on graph pooling—a representative task in this domain—and propose CurvPool, a method that integrates curvature information into the pooling process. CurvPool adjusts node connectivity in the original graph by modifying edge weights based on their curvature values: edges with positive curvature have their weights increased to pull nodes closer, while those with negative curvature have their weights decreased to push nodes apart. The initial graph is assumed to have uniform edge weights of 1. Implementation details are provided in the 
Appendix~\ref{a_Hyperparameter}.

We evaluate two variants of our method—CurvPool-O (with Ollivier-Ricci curvature) and CurvPool-R (with effective resistance curvature)—on five benchmark datasets. Experimental settings are summarized in the 
Appendix~\ref{a_Hyperparameter}, and results are presented in Table~\ref{tab:graph_pool_results}.

The results lead to the following conclusions:
\begin{itemize}
    \item CurvPool consistently outperforms baseline methods across nearly all datasets, demonstrating the general utility of curvature-aware pooling.
    \item The performance of CurvPool-O and CurvPool-R is comparable, reaffirming that effective resistance curvature is a competitive and efficient alternative to Ollivier-Ricci curvature in graph-structured learning tasks.
\end{itemize}

\subsection{Efficiency Comparison Experiments}

\begin{table*}[!htbp]
\centering
\caption{Curvature computation time cost (seconds) comparison in real datasets}
\label{tab:real_cost_results}
\begin{tabular}{lccccccc}
\hline
Device & Cora & Citeseer & PubMed &  Amazon Computers & Amazon Photo  & Coauthor CS & Coauthor Physics \\
\hline
-O(CPU 1) & 3.564 & 3.784 & 70.666  & 896.576 & 190.561 & 78.949 & 370.953 \\
-O(CPU 2) & 4.881 & 5.301 & 66.392  & 754.096 & 169.635 & 103.256 & 333.013 \\
\hline
-R L20 & 0.037 & 0.035 & 1.635  & 0.683 & 0.188 & 1.080 & 5.926 \\
-R 3090ti & 0.052 & 0.045 & 1.753  & 0.725 & 0.198 & 1.444 & - \\
\hline
Maximum Speedup Ratio & 132x & 151x & 43x  & 1313x & 1032x & 93x & 63x \\
\hline
\end{tabular}
\\
\footnotesize \textit{Note:} Dash (-) indicates exceeding the computational capabilities.
\end{table*}

To assess the computational efficiency of Ollivier-Ricci curvature and effective resistance curvature, we conducted runtime comparisons using both real-world and synthetic datasets from the node classification experiments.

In the real-world dataset experiments, we measured the computation time of both curvature methods on seven real-world datasets under two CPU configurations (CPU1: Intel Xeon Gold 6248R, 96 threads; CPU2: Intel Xeon E5-2603 v4, 24 threads) and two GPU environments (GPU1: NVIDIA RTX 3090 Ti; GPU2: NVIDIA L20). All CPU trials used full-thread parallelization, while GPU tests used a single device. Each dataset was run 10 times, and the average runtime was recorded. Results are summarized in Table~\ref{tab:real_cost_results}.

As shown in the table, effective resistance curvature is significantly faster than Ollivier-Ricci curvature across all datasets, with a maximum speedup of 1313× observed on the Amazon Computers dataset (CPU2: 896.576 s vs. GPU2: 0.683 s).

Notably, the time complexity of Ollivier-Ricci curvature depends mainly on the number of edges, whereas effective resistance curvature scales with the number of nodes. This distinction is illustrated by the PubMed and Amazon Photo datasets: although PubMed has more nodes (19,717 vs. 7,650), Amazon Photo has more edges, leading to longer Ollivier-Ricci curvature computation times (169.635 s vs. 66.392 s on CPU2). In contrast, effective resistance curvature is faster on Amazon Photo due to its smaller node set (0.198 s vs. 1.753 s on GPU2).

To further examine scalability, we generated graphs using the Newman–Watts (NW) and random regular (RR) models with node counts ranging from 1,000 to 10,000 and node degrees set to \{10, 30, 50\}. Each configuration was executed 10 times on CPU1 and GPU1, with average runtimes reported in Fig.~\ref{fig:generated_cost_tendency}. The results confirm that Ollivier-Ricci curvature requires substantially more time than effective resistance curvature as graph size increases, clearly demonstrating the computational advantage of the latter.

\begin{figure}[!htbp]
	\centering 
	\includegraphics[width=0.5\textwidth]{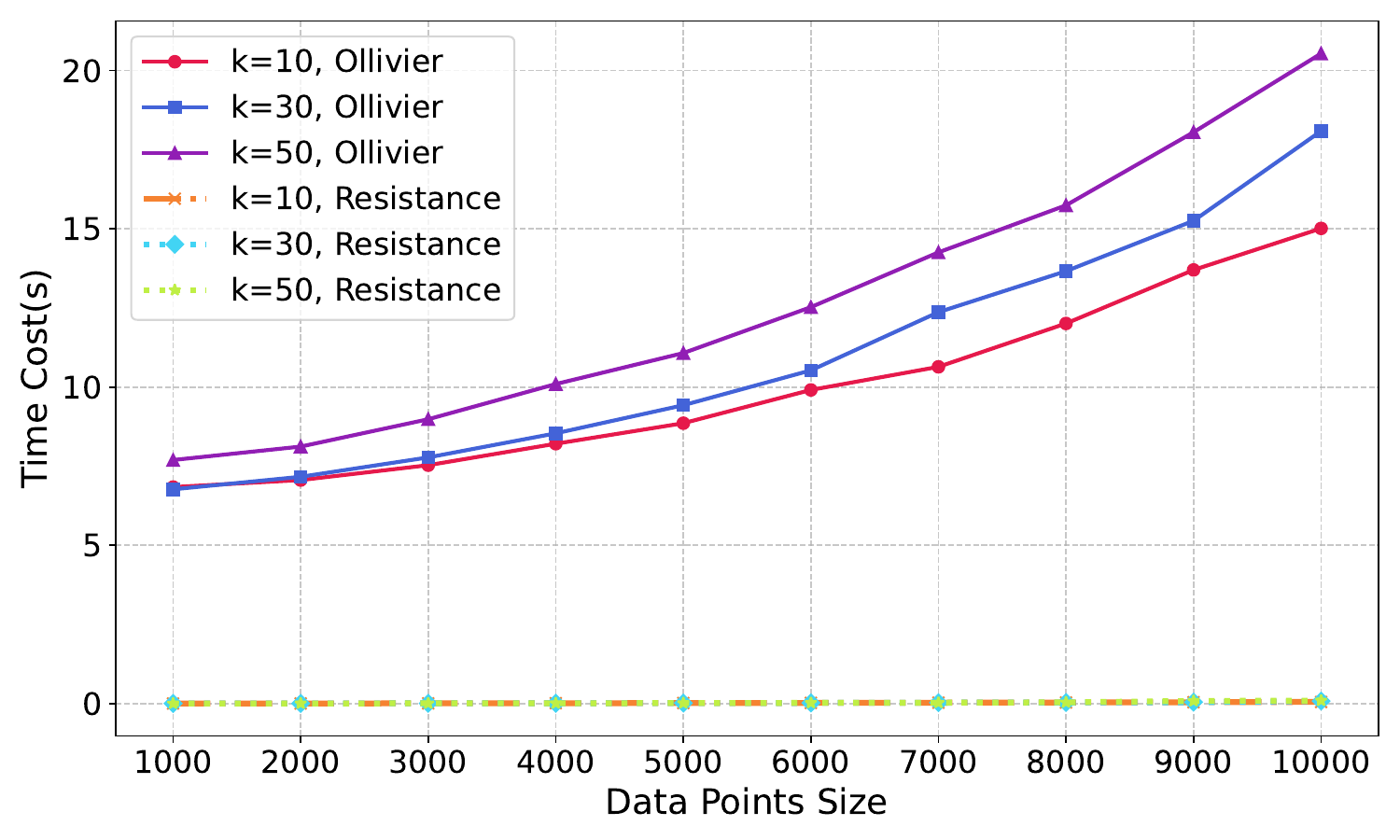} 
    \includegraphics[width=0.5\textwidth]{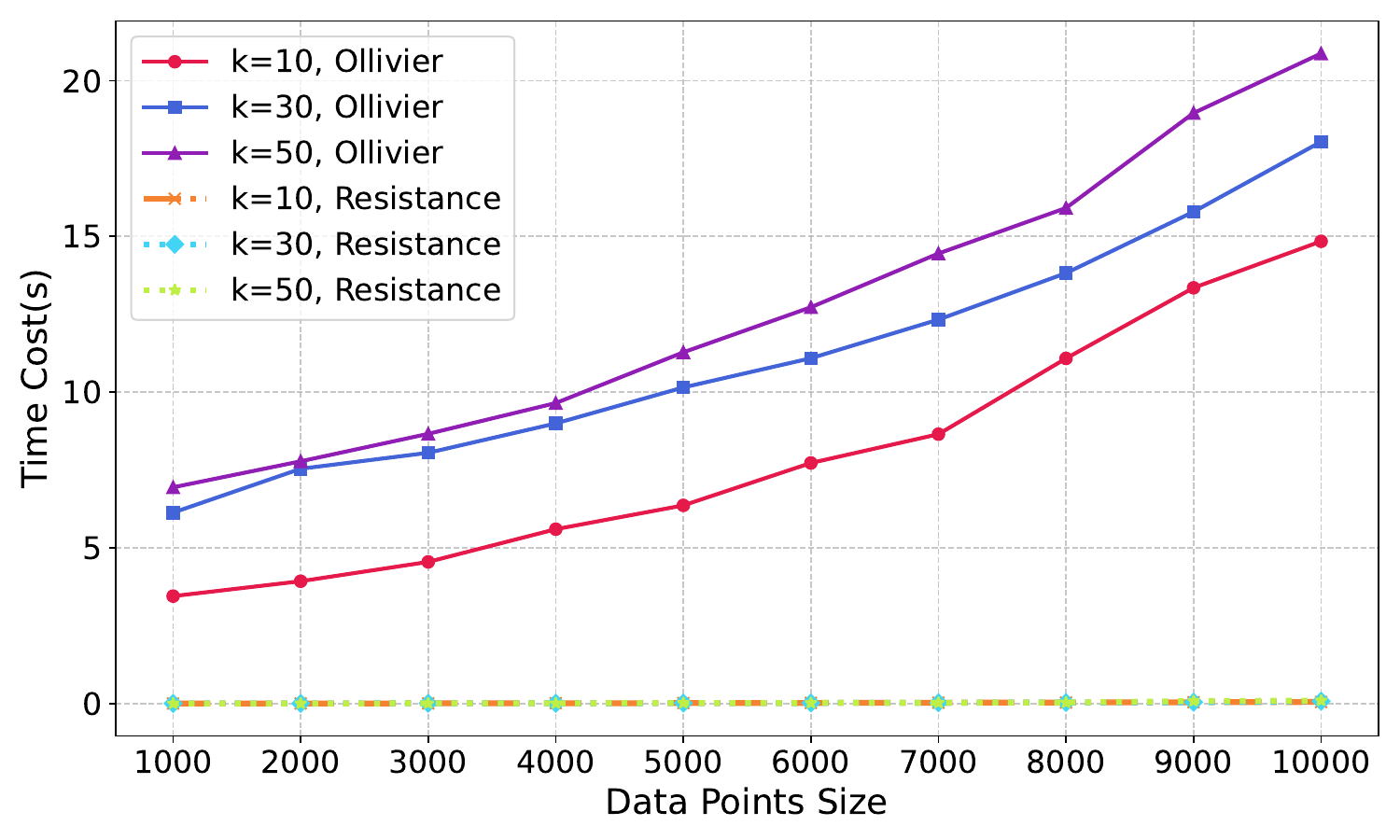} 
	\caption{Time cost comparison for Ollivier-Ricci and effective resistance curvature on NW(upper) and RR (lower)) graph}.
	\label{fig:generated_cost_tendency}
\end{figure}

\subsection{An Empirical Comparative Analysis of Effective Resistance Curvature and Ollivier-Ricci Curvature: Consistency and Divergence}

Building upon the theoretical analysis in section~\ref{relation_between_O_R}, this section further investigates the consistency and differences between effective resistance curvature and Ollivier-Ricci curvature through a series of experiments on real-world datasets. The experiments are conducted primarily along two dimensions: (1) the distribution characteristics of the two curvatures in real datasets; (2) their correlation patterns with graph geometric properties.

\subsubsection{Comparison of Curvature Distribution Characteristics}
We analyzed the distributions of the two curvatures on three citation network datasets: Cora, CiteSeer, and PubMed. As shown in Fig.~\ref{fig:curvature_dist_KDE}, the results reveal the following characteristics:

\begin{figure*}[htbp]
	\centering 
    \includegraphics[width=0.32\textwidth]{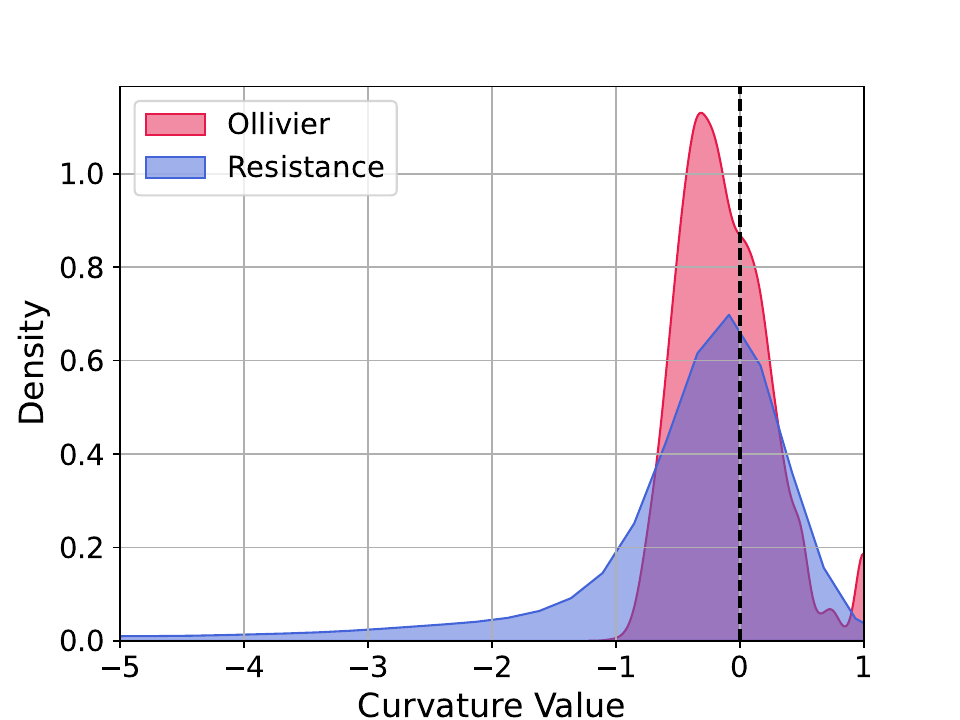} 
	\includegraphics[width=0.32\textwidth]{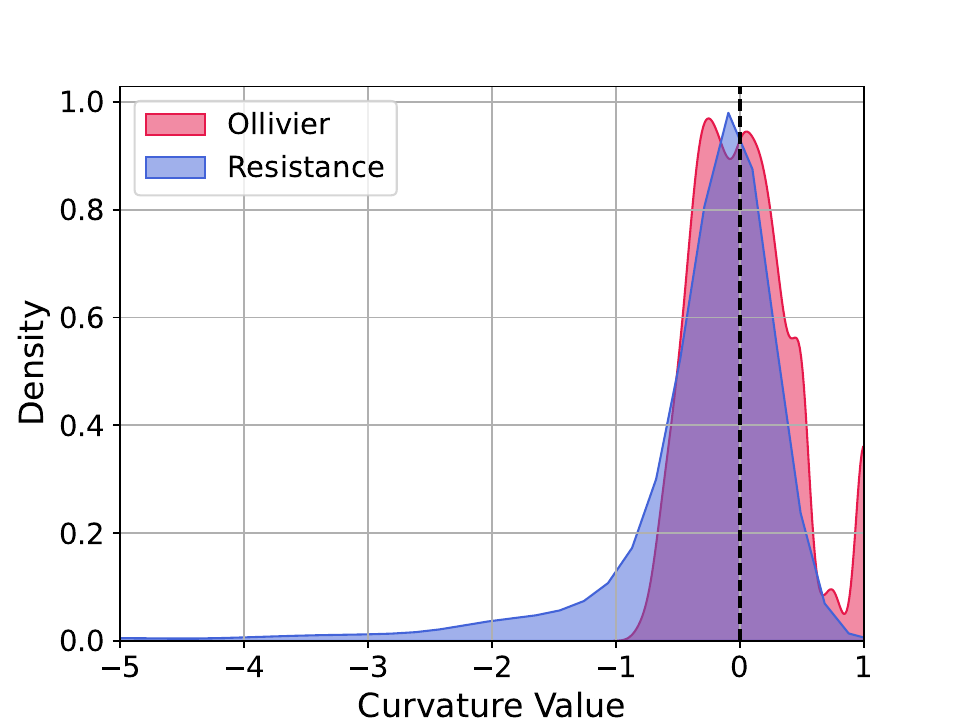} 
    \includegraphics[width=0.32\textwidth]{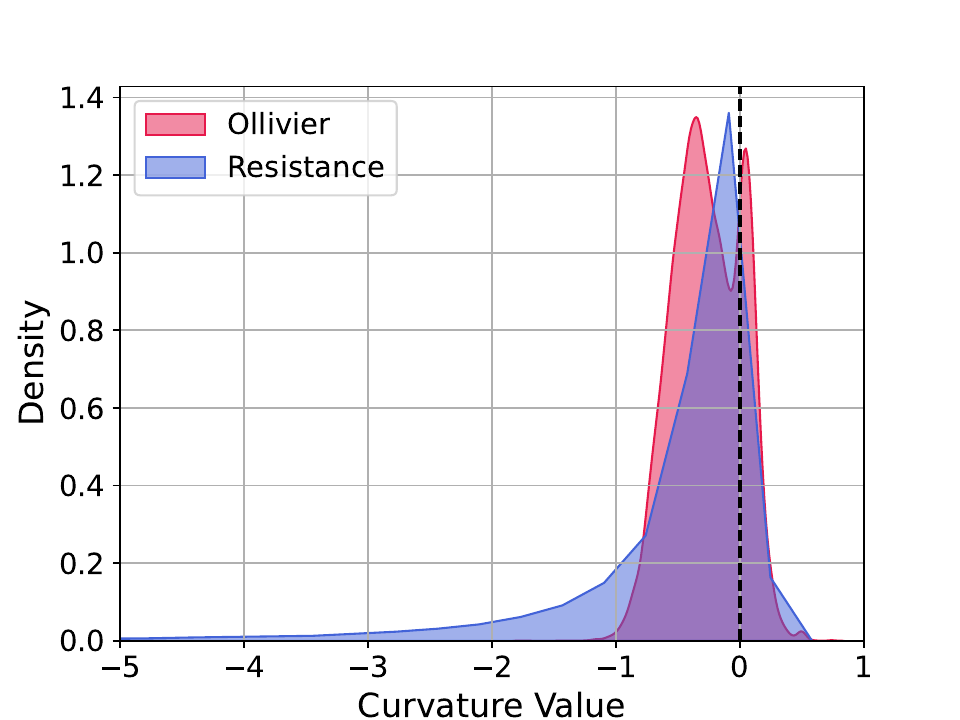} 
	\caption{Curvature distribution KDE on three real-world dataset(left: Cora, center: Citeseer, right: PubMed)}.
	\label{fig:curvature_dist_KDE}
\end{figure*}

\begin{figure*}[htbp]
	\centering 
    \includegraphics[width=0.32\textwidth]{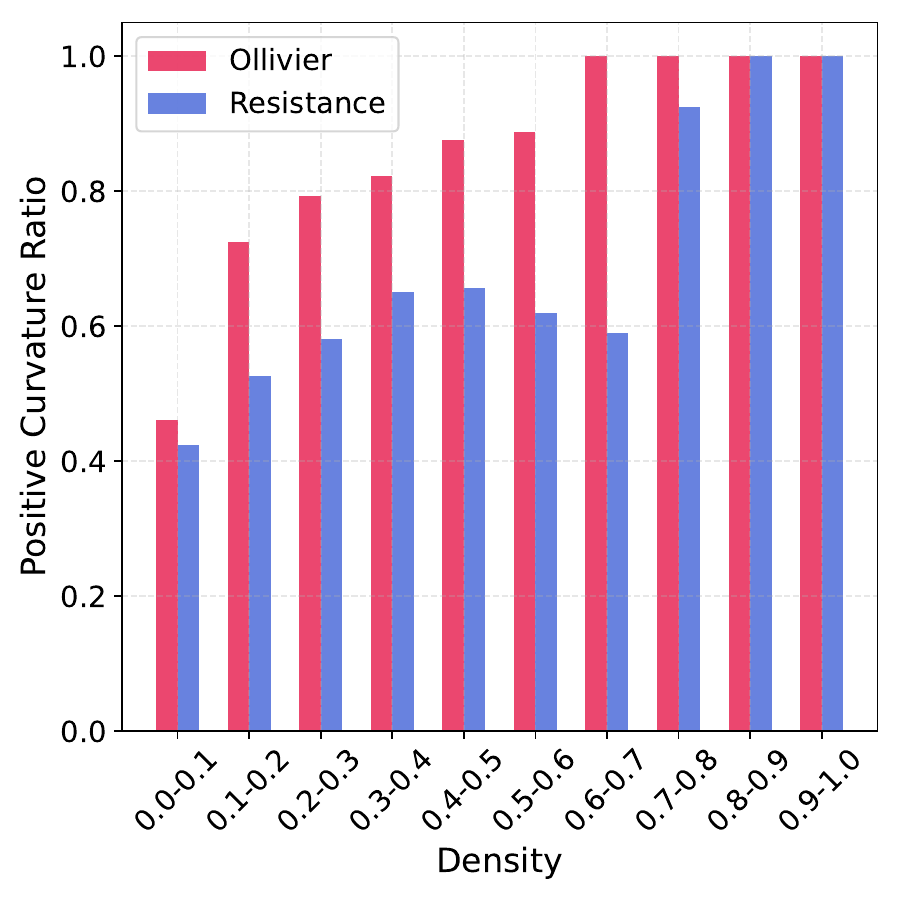} 
	\includegraphics[width=0.32\textwidth]{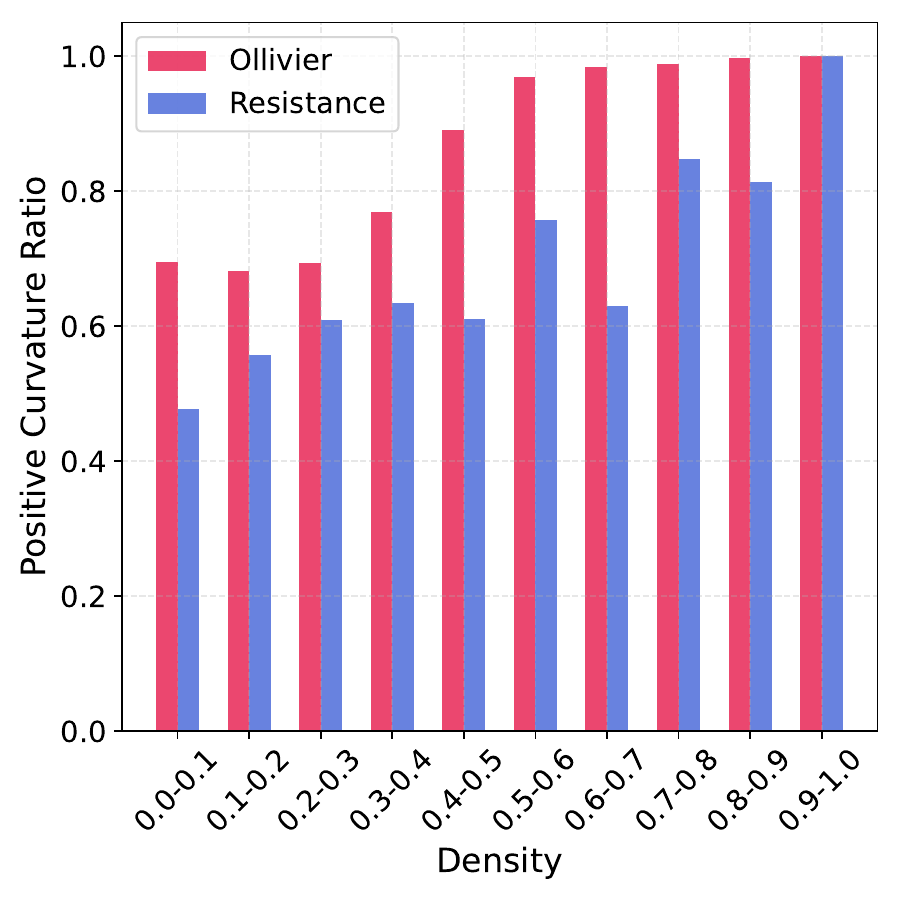} 
    \includegraphics[width=0.32\textwidth]{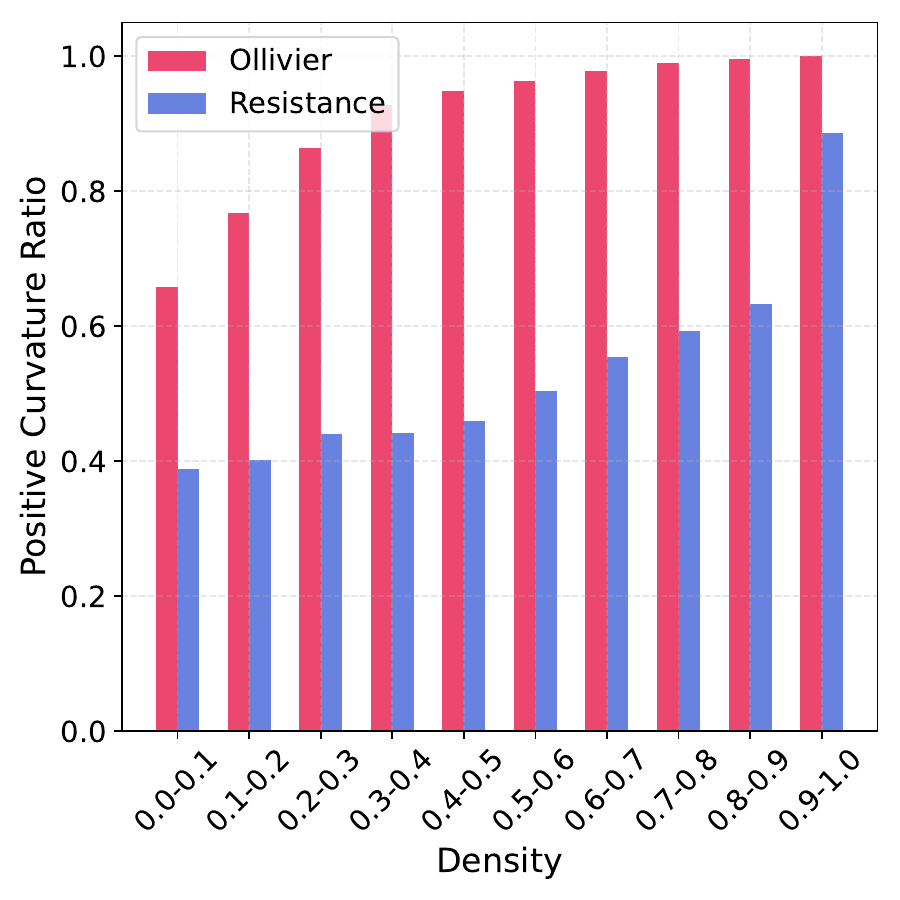} 
	\caption{Proportion Distribution of Positively-Curved Edges on Three Real Datasets (left: PROTEINS, center: IMDB-B, right: COLLAB)}.
	\label{fig:postive_ratio_density}
\end{figure*}

\textbf{Consistency in distribution center}: Across all datasets, both curvature values are densely distributed near 0. This indicates that the overall geometric structure of these citation networks is close to Euclidean space, with curvature values for most edges or node pairs approaching zero, reflecting an overall flat characteristic of the networks.

\textbf{Divergence in distribution shape}: The Ollivier-Ricci curvature distribution exhibits an unstable shape, being right-skewed (Cora, CiteSeer) or left-skewed (PubMed). In contrast, the effective resistance curvature maintains a stable left-skewed distribution across all datasets, demonstrating stronger robustness.

\textbf{Difference in capturing anomalous structures}: The tails of the Ollivier-Ricci curvature distribution decay rapidly, resulting in a sharp and narrow shape that identifies fewer significant positive/negative curvature anomalies. Conversely, effective resistance curvature exhibits a distinct "heavy-tail" characteristic, with a flatter and broader distribution. This suggests its enhanced capability to capture local structures with significant curvature values and greater sensitivity to subtle geometric differences.

\subsubsection{Correlation Analysis with Graph Geometric Properties}
To further explore the relationship between the two curvatures and graph geometric properties, experiments were conducted on five datasets: ENZYMES and PROTEINS (sparse graphs), as well as IMDB-B, IMDB-M, and COLLAB (dense graphs). The analysis focused on two key metrics: Graph Density and Betweenness Centrality. Betweenness Centrality measures the frequency with which a node lies on the shortest paths between other node pairs and is a core metric for identifying hub nodes in a network.

The experiments revealed the following patterns:

\textbf{Consistent relationship with community size}: As shown in Fig.~\ref{fig:fully-connected-curvature}, both curvatures decrease monotonically with increasing community size (Pattern 1) and eventually converge. Validation experiments on complete graphs show that the upper bounds for both effective resistance curvature and Ollivier-Ricci curvature are 1, while their lower bounds are 0 and 0.5, respectively.

\begin{figure}[htbp]
	\centering 
	\includegraphics[width=0.5\textwidth]{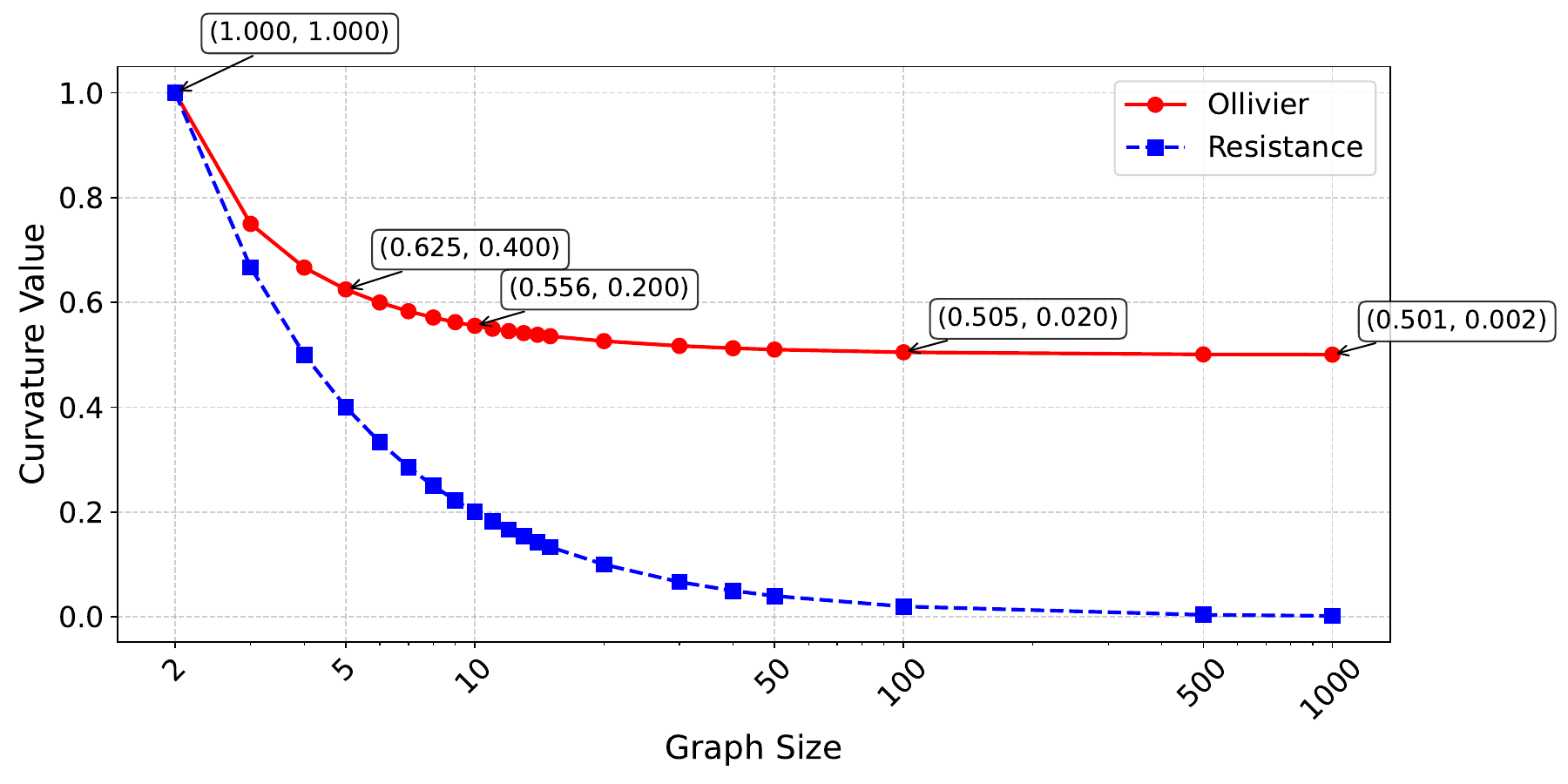} 
	\caption{Ollivier-Ricci and effective resistance curvature value in different graph size on fully-connected graph.}
	\label{fig:fully-connected-curvature}
\end{figure}

\textbf{Consistent relationship with graph density}: As shown in Fig.~\ref{fig:postive_ratio_density}, the proportion of edges with positive curvature for both measures increases significantly with rising graph density, indicating that positive curvature relationships dominate in high-density graphs.

\textbf{Divergent association with high-betweenness-centrality nodes}: For effective resistance curvature, the larger the size of the clique connected to a high-betweenness-centrality node, the smaller the curvature value of the edges between that node and the nodes within the clique (Pattern 2-1). For Ollivier-Ricci curvature, however, the larger the clique size connected to a high-betweenness-centrality node, the largerthe curvature value of the corresponding edges (Pattern 2-2).

To validate these patterns, we employed the Girvan-Newman algorithm for community detection (which iteratively removes edges with the highest edge betweenness to reveal community structure) and statistically analyzed graph structures containing only a single high-betweenness-centrality node (to exclude interference from multiple central nodes). Statistical results across the five datasets (see Table~\ref{tab:pattern212}) confirm that Patterns 2-1 and 2-2 hold in the vast majority of cases.

\subsubsection{Further Interpretation of Key Divergence}
Both effective resistance curvature and Ollivier-Ricci curvature can effectively identify high-betweenness-centrality nodes. However, effective resistance curvature demonstrates a wider identification range; even in ultra-dense graphs, it can distinguish "bottleneck" nodes based on the sign of the curvature value. For instance, in the example shown in the 
Fig.~\ref{fig:examp_imdb_b_42}. 

effective resistance curvature assigns negative curvature to the connections within the set of nodes $C_4$, which has the highest betweenness centrality, whereas Ollivier-Ricci curvature assigns them positive values, highlighting their fundamental difference.

Interpreting the sign of curvature as indicative of "repulsive" (negative curvature) or "attractive" (positive curvature) relationships between nodes, effective resistance curvature tends to characterize relationships among high-betweenness-centrality nodes as competitive or repulsive. In contrast, Ollivier-Ricci curvature tends to portray them as attractive or complementary. This fundamental difference makes them suitable for different scenarios: effective resistance curvature excels at detecting fragile connections or bottleneck structures within networks, while Ollivier-Ricci curvature is more adept at revealing cooperative relationships within functional modules (see the 
Appendix~\ref{Divergence_in_Application} for examples).

\setlength{\tabcolsep}{4pt}
\begin{table}[!htbp]
\centering
\caption{The fulfillment status of Pattern 2-1 and Pattern 2-2 across five real-world datasets.}
\label{tab:pattern212}
\begin{tabular}{lccccc}
\hline
Dataset & ENZYMES & PROTEINS & IMDB-B & IMDB-M &COLLAB \\
\hline
Pattern 2-1 & 100 & 96.1 & 100  & 100 & 100 \\
Pattern 2-2 & 100 & 96.1 & 94.54  & 96.83 & 83.75  \\
\hline
\end{tabular}
\end{table}

\section{Conclusion}\label{sec:conc}

Mining nonlinear curvature information in graph structures is essential for advancing graph learning beyond conventional topological analysis and achieving more accurate geometric characterization. While Ollivier-Ricci curvature offers a solid theoretical foundation for capturing graph geometry, its practical application faces significant challenges due to prohibitive computational complexity and pronounced sensitivity to hyperparameter selection. To overcome these limitations, we introduce effective resistance curvature, a novel metric grounded in circuit network theory that conceptualizes graphs as resistor networks. This approach naturally captures connectivity patterns and bottleneck structures while maintaining compelling geometric expressiveness. More importantly, it offers superior computational efficiency and enhanced robustness to structural perturbations compared to optimal transport-based methods. Extensive experiments across diverse graph learning tasks demonstrate that our method achieves performance comparable to Ollivier-Ricci curvature while delivering substantially improved scalability and numerical stability. These advantages position effective resistance curvature as a practical and scalable tool for curvature-based analysis on large-scale real-world graphs.

\section{Appendix}
\subsection{Ollivier-Ricci Curvature}
Ollivier-Ricci curvature is a metric used in graph theory and metric spaces to describe the geometric properties of a space. It quantifies the "curvature" between two points using the Wasserstein distance, offering a way to measure how the distance between points changes in a graph or network. This concept was first introduced by Ollivier in 2009, aiming to provide a curvature measure for networks and discrete graphs that is analogous to Ricci curvature in Riemannian geometry.

For a weighted graph $G=(V,E)$ with a metric $d$ defined on node pairs, where $V$ is the set of nodes and $E$ is the set of edges, Ollivier-Ricci curvature provides a curvature measure between two nodes. Let $x, y\in V$ be any two nodes in a graph, and let $u$ and $v$ represent two probability distributions starting from $x$ and $y$, respectively. The definition of Ollivier-Ricci curvature is based on the variation of the Wasserstein distance between these two probability distributions, expressed as:
\begin{equation}
    k(x,y)= 1-\frac{W_1(u,v)}{d(x,y)},
\end{equation}
where $W_1(u,v)$ denotes the Wasserstein-1 distance between $u$ and $v$, and $d(x,y)$ is the graph distance between nodes $x$ and $y$. This curvature quantifies how the probability mass spreads between the two distributions relative to their graph distance.

\subsection{Supplementary Experimental Settings}
\label{a_Hyperparameter}

\textbf{Hyperparameter settings for graph node classification on real-world datasets.} 
 For the real world dataset, we follow CGNN [14] which adjusting the hidden layer dimension of CGNNs to 64. And the architectures of baselines are consistent with their corresponding papers. For both CGNN and CurvGN, we use Adam SGD optimizer with a learning rate of 0.005 and L2 regularization of 0.0005, and cross-entropy as the loss function to train the models. In this paper, the weight matrix is initialized with Glorot initialization. 

\textbf{Hyperparameter settings for graph node classification on Synthetic datasets. }
For the synthetic datasets, we follow CurvGN [14] which all three models, Watts-Strogatz, Newman-Watts and Kleinberg’s model, are created by randomly modifying/adding edges to a ring graph. A ring graph has n nodes embedded on a circle, with each node connected to its k nearest neighbors. 

For both CGNN and CurvGN, we use Adam SGD optimizer with all learning rates and regularization rates are set to 0.01 and 0.0005. For dense graphs, the number of neighbors is designed to be 20, and for sparse graphs, the number of neighbors is 5. The graph has 1000 nodes. For GAT, we set the dimension of the hidden layer in models to 8 and set the head of GAT to 1.

\textbf{Hyperparameter settings for graph pooling on real-world datasets. }
The original MinCut paper did not report results on the ENZYMES dataset. As an important baseline, we performed a grid search to re-tune the learning rate and regularization rate, obtaining the optimal performance of MinCut on the ENZYMES dataset. Using these hyperparameters, we further derived the results for CurvPool-O and CurvPool-R.

We adpot formula~\ref{eq:graph_pool} to adjust orginal weights in a graph to a new weight by suitable curvature intervention. hyperparameter $\eta$ is used to control the degree of curvature intervention. All the best $\eta$ values acheiving best performance for all datasets are shown in Table~\ref{tab:graph_pool_setting1}.
\begin{equation}
\label{eq:graph_pool}
w_{ij}=w_{ij}(1-\eta*k_{ij})
\end{equation}

\begin{table*}[!htbp]
    \centering
    \caption{Graph curvature learning rate used in various datasets in graph pooling experiments}
    \label{tab:graph_pool_setting1}
    \begin{tabular}{lccccccc}
        \hline
        Dataset & IMDBMULTI & COLLAB & PROTEINS & DD & ENZYMES & MUTAG & IMDBBINARY \\
        \hline
        CurvPool-O & 0.3 & 0.4 & 0.3 & 0.6 & 0.1 & 0.9 & 0.2 \\
        CurvPool-R & 0.5 & 0.6 & 0.8 & 0.9 & 1.0 & 0.2 & 0.9 \\
        \hline
    \end{tabular}
\end{table*}

\subsection{Details of Graph Model}
\label{a_graphmodel}
\begin{itemize}
    \item \textbf{Watts–Strogatz}: The Watts-Strogatz (WS) model is a random graph model used to generate networks with the small-world property. It begins with a ring-structured regular graph in which each node is connected to its K nearest neighbors; each edge is then randomly rewired with probability p. This model is widely used for modeling real-world systems such as social networks, neural networks, infectious disease transmission, and power grids.
    \item \textbf{Newman–Watts}: The Newman-Watts (NW) model is a variant of the Watts-Strogatz model. Unlike the WS model, which modifies existing edges through “rewiring,” the NW model directly introduces random shortcuts while preserving the original regular connections, thereby avoiding issues such as network fragmentation or isolated nodes that may arise from the rewiring process.
    \item \textbf{Kleinberg’s} navigable small world graph: The Kleinberg model extends the WS model by incorporating long-range connections in addition to local links. These long-range connections are not entirely random but are generated based on a power-law distribution related to grid distance ($P(\text{link}) \propto d^{-\alpha}$), providing a theoretical foundation for the design of distributed systems, P2P networks, and social network search mechanisms.
    \item \textbf{Stochastic Block Model}: The Stochastic Block Model (SBM) is a generative model for random graphs with community structure. Nodes are assigned to different communities (blocks), and the probability of a connection between any two nodes depends entirely on their community assignments. Typically, the within-community connection probability ($p_{in}$) is much higher than the between-community connection probability ($p_{out}$), making SBM a fundamental and essential generative model in the study of community detection algorithms.
    \item \textbf{Random Regular} graph: A random regular graph is a graph uniformly selected from all possible graphs where every vertex has exactly the same degree k, and its specific connections are formed randomly.  
\end{itemize}

\subsection{Examples for Pattern 2-1 and Pattern 2-2}
\label{a_geometric_analysis}
\begin{figure}[htbp]
	\centering 
	\includegraphics[width=0.5\textwidth]{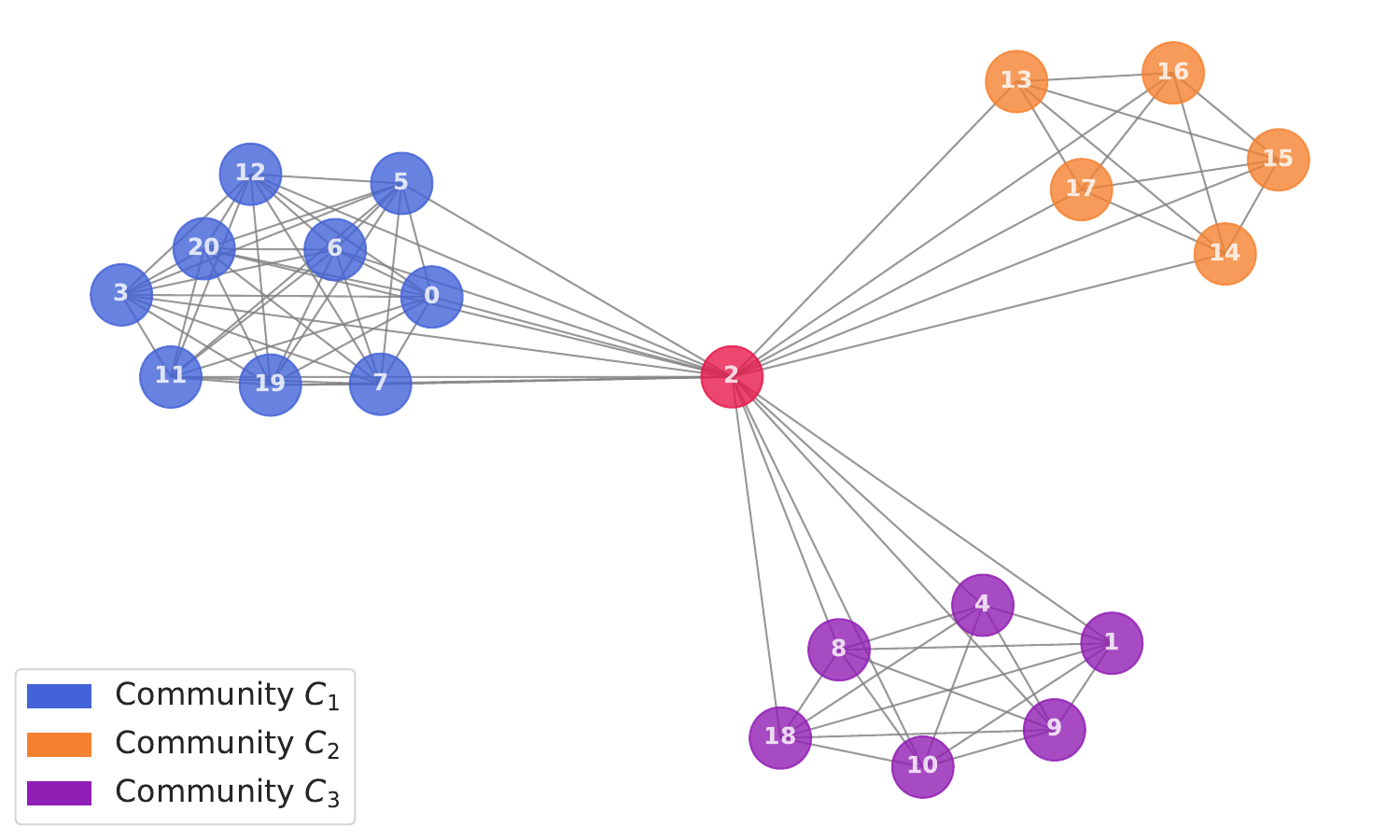} 
	\caption{Sparse graph example from dataset IMDB-B (source graph id is 2), its density is 0.386.}
	\label{fig:examp_imdb_b_2}
\end{figure}

\begin{figure}[htbp]
	\centering 
	\includegraphics[width=0.5\textwidth]{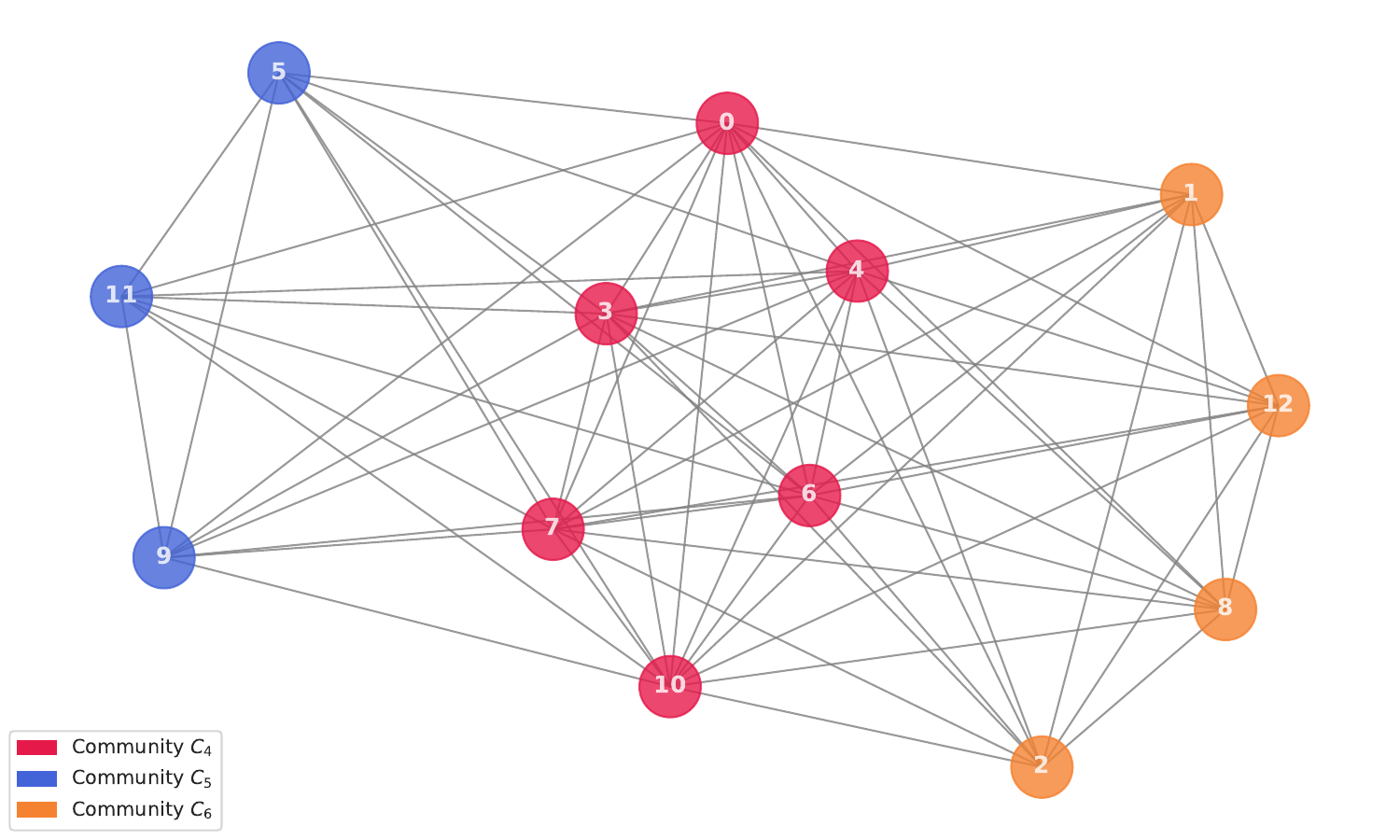} 
	\caption{Dense graph example from dataset IMDB-B (source graph id is 42), its density is 0.846.}
	\label{fig:examp_imdb_b_42}
\end{figure}

We illustrate this using two real-data examples: Fig.~\ref{fig:examp_imdb_b_2} (a sparse graph) and Fig.~\ref{fig:examp_imdb_b_42} (a dense graph).
In Fig.~\ref{fig:examp_imdb_b_2}, Node 2 has the highest betweenness centrality. The sizes of communities $C_1$=\{13,14,15,16,17\}, $C_2$=\{1,4,8,9,10,18\}, and $C_3$=\{0,3,5,6,7,11,12,19,20\} are 5, 6, and 9, respectively. The resistance curvature values between nodes within these three communities are 0.40042, 0.40037, and 0.40031, respectively, satisfying the pattern that curvature decreases with increasing community size. Similarly, the Ollivier-Ricci curvature values between nodes within these communities are 0.6, 0.583, and 0.556, also conforming to the pattern of decreasing curvature with larger community size. However, the differences between the resistance curvature values are relatively small, whereas the differences between the Ollivier-Ricci curvature values are more pronounced, indicating that Ollivier-Ricci curvature possesses a stronger capability to distinguish between communities of different sizes.

The resistance curvature values for edges between Node 2 and nodes in communities $C_1$, $C_2$, and $C_3$ are -1.707, -2.025, and -2.978, respectively. This demonstrates that the larger the community connected to the node, the smaller the curvature value, consistent with Pattern 2-1. Conversely, the Ollivier-Ricci curvature values for edges between Node 2 and nodes in these three communities are -0.15, -0.1167, and 0.0056, respectively. Here, the larger the connected community, the largerthe curvature value, consistent with Pattern 2-2. Notably, the magnitude of difference between the resistance curvature values is substantial, while the differences between the Ollivier-Ricci curvature values are relatively minor. This suggests that resistance curvature exhibits a stronger capability to distinguish between cliques of different sizes.

Consider another example involving an ultra-dense graph, as shown in Fig.~\ref{fig:examp_imdb_b_42}. The sizes of communities $C_4$=\{0,3,4,5,6,7,10\}, $C_5$=\{1,2,8,12\}, and $C_6$=\{5,9,11\} in this graph are 7, 4, and 3, respectively. The Ollivier-Ricci curvature values for edges within these three communities are 0.5417, 0.5556, and 0.5625, showing a gradual increase as community size decreases, which aligns with Pattern 1. Similarly, the resistance curvature values for edges within these communities are -0.1465, 0.4855, and 0.5446, also increasing as community size decreases, consistent with Pattern 1.
Ultra-dense graphs often contain multiple high-betweenness-centrality nodes. In this graph, all nodes in community $C_4$ are high-betweenness-centrality nodes. The resistance curvature values for edges between nodes in $C_4$ and nodes in $C_5$ and $C_6$ are 0.206 and 0.254, respectively. This satisfies Pattern 2-1: the larger the clique connected to the high-betweenness-centrality node, the smaller the curvature value. In contrast, the Ollivier-Ricci curvature values for edges between nodes in $C_4$ and nodes in $C_5$ and $C_6$ are 0.375 and 0.33, respectively. This satisfies Pattern 2-2: the larger the connected clique, the largerthe curvature value.

\textbf{Pattern 3}: The resistance curvature and Ollivier-Ricci curvature of edges connected to nodes with the highest betweenness centrality exhibit different behaviors. Specifically, the resistance curvature of edges connected to the node with the highest betweenness centrality is always negative (Pattern 3-1), while the Ollivier-Ricci curvature of such edges is always positive (Pattern 3-2).

Pattern 3 can be regarded as an extension of Pattern 2. We evaluated the extent to which these two sub-patterns hold across the five datasets. As shown in Table~\ref{tab:pattern312}, both patterns are strongly supported in all five datasets. For resistance curvature, the lower the graph density, the higher the probability that Pattern 3-1 holds—i.e., in sparser graphs, the resistance curvature of edges connected to the node with the highest betweenness centrality is more likely to be negative. For example, in IMDB-B, when the graph density is below 0.71, all graphs satisfy Pattern 3-1. For Ollivier-Ricci curvature, the higher the graph density, the greater the probability that Pattern 3-2 holds—i.e., in denser graphs, the Ollivier-Ricci curvature of edges connected to the node with the highest betweenness centrality is more likely to be positive. For instance, in IMDB-B, when the graph density exceeds 0.83, all graphs satisfy Pattern 3-2.

\setlength{\tabcolsep}{9pt}
\begin{table}[!htbp]
\centering
\caption{The fulfillment status of Pattern 3-1 and Pattern 3-2 across five real-world datasets.}
\label{tab:pattern312}
\begin{tabular}{c|cc|cc}
\hline
\multirow{2}{*}{Dataset} & \multicolumn{2}{c|}{Pattern 3-1} & \multicolumn{2}{c}{Pattern 3-2} \\
\cline{2-5}
 & Density & Ratio & Density & Ratio \\
\hline
\multirow{2}{*}{ENZYMES} & $<$1.00 & 91.31 & $>$0.50 & 93.75 \\
\cline{2-5}
 & $<$0.05 & 100.00 & $>$0.54 & 100.00 \\
\hline
\multirow{2}{*}{PROTEINS} & $<$1.00 & 84.83 & $>$0.50 & 95.15 \\
\cline{2-5}
 & $<$0.15 & 91.37 & $>$0.55 & 100.00 \\
\hline
\multirow{2}{*}{IMDB-B} & $<$1.00 & 98.43 & $>$0.50 & 88.53 \\
\cline{2-5}
 & $<$0.71 & 100.00 & $>$0.83 & 100.00 \\
\hline
\multirow{2}{*}{IMDB-M} & $<$1.00 & 97.01 & $>$0.50 & 93.57 \\
\cline{2-5}
 & $<$0.83 & 100.00 & $>$0.85 & 100.00 \\
\hline
\multirow{2}{*}{COLLAB} & $<$1.00 & 97.84 & $>$0.50 & 89.54 \\
\cline{2-5}
 & $<$0.65 & 99.96 & $>$0.65 & 99.93 \\
\hline
\end{tabular}
\end{table}

\textbf{Divergence in the Applications of Ollivier-Ricci and Resistance Curvatures}
\label{Divergence_in_Application}
High-betweenness nodes are typically located on connecting paths between network clusters (e.g., "bridges" or "bottlenecks") and can serve as split points for cluster partitioning. Both resistance curvature and Ollivier-Ricci curvature can identify such high-centrality nodes, but resistance curvature offers a broader identification range. Even in relatively dense graphs, it can still detect bottleneck nodes through curvature signs. As shown in Fig.~\ref{fig:examp_imdb_b_42}, resistance curvature yields negative curvature within the set of highest-betweenness nodes $C_4$, whereas Ollivier-Ricci curvature results in positive curvature among these nodes, highlighting a significant distinction between the two. This indicates that, even in ultra-dense graphs, resistance curvature can still distinguish node relationships based on the sign of curvature.

If the sign of edge curvature is used to measure the "repulsiveness" or "attractiveness" of a connection, negative curvature suggests a competitive or repulsive relationship between nodes, while positive curvature indicates an attractive or complementary relationship. In resistance curvature, bottleneck nodes and nodes in other communities are often assigned negative curvature, so high-betweenness nodes typically exhibit repulsive relationships. In contrast, Ollivier-Ricci curvature tends to assign positive curvature to relationships with the largest communities, thereby reflecting attractive interactions. This fundamental difference makes each curvature suitable for different application scenarios, with each having its own advantages.

Resistance curvature introduces negative curvature between high-betweenness nodes, providing benefits in the following scenarios:1) \textbf{Suppressing rumor propagation}: If high-influence nodes (e.g., "big V" users) exhibit positive curvature, information may spread rapidly but can also facilitate rumor diffusion. Negative curvature hinders direct information transmission through such nodes, forcing information to detour through ordinary users (via positive curvature paths), thereby lengthening the propagation path and buying time for intervention. 2) \textbf{Optimizing resource allocation and enhancing disaster resilience}: In transportation/communication networks, overloaded hub nodes (e.g., core routers) can cause network-wide congestion. Negative curvature (high transmission cost) between nodes diverts traffic to secondary paths, avoiding single-point overload.

Ollivier-Ricci curvature, on the other hand, tends to establish positive curvature between high-betweenness nodes, strengthening their connections. It is suitable for the following scenarios:1) \textbf{Strengthening information integration and group cohesion}: If high-centrality nodes are isolated, information fragmentation and group polarization may occur. Positive curvature enables key nodes (e.g., media, government, NGOs) to rapidly share information and coordinate resources during disasters. 2) Improving transmission efficiency and disaster resilience: Positive curvature (low transmission cost) facilitates automatic traffic balancing between core nodes, preventing single-point overload. For example, positive curvature links between 5G base stations enable cooperative relaying of user data across multiple base stations, enhancing transmission robustness.

In summary, resistance curvature suppresses risks and distributes load through negative curvature, while Ollivier-Ricci curvature enhances collaboration and improves efficiency via positive curvature.
\end{document}